\documentclass{article}

\usepackage[margin=1in]{geometry}
\usepackage[utf8]{inputenc}
\usepackage{amsmath}
\usepackage{amssymb}
\usepackage{amsthm}
\usepackage{xargs}
\usepackage{bm}
\usepackage{graphicx}
\usepackage{booktabs}
\usepackage{tabularx}
\usepackage[colorlinks=true, linkcolor=blue]{hyperref}
\usepackage{scalerel,stackengine}
\usepackage{xcolor}
\usepackage{multirow}
\usepackage{makecell}
\usepackage{chngcntr}
\usepackage{apptools}
\usepackage{dsfont}
\usepackage{authblk}
\usepackage[numbers]{natbib}
\usepackage{longtable}
\AtAppendix{\counterwithin{theorem}{section}}

\newtheorem{theorem}{Theorem}
\newtheorem{lemma}[theorem]{Lemma}

\newtheorem{corollary}{Corollary}[theorem]

\theoremstyle{definition}
\newtheorem{definition}{Definition}[section]

\theoremstyle{remark}

\newcommand{\numeral}[1]{%
  \textup{\uppercase\expandafter{\romannumeral#1}}%
}

\stackMath
\newcommand\reallywidehat[1]{%
\savestack{\tmpbox}{\stretchto{%
  \scaleto{%
    \scalerel*[\widthof{\ensuremath{#1}}]{\kern.1pt\mathchar"0362\kern.1pt}%
    {\rule{0ex}{\textheight}}%WIDTH-LIMITED CIRCUMFLEX
  }{\textheight}% 
}{2.4ex}}%
\stackon[-6.9pt]{#1}{\tmpbox}%
}
\title{
Increasing the efficiency of randomized trial estimates via linear adjustment for a prognostic score
}

\author[1]{Alejandro Schuler\thanks{aschuler@unlearn.ai}}
\author[1]{David Walsh}
\author[1]{Diana Hall}
\author[1]{Jon Walsh}
\author[1]{Charles Fisher}
\affil[1]{Unlearn.AI, Inc., San Francisco, CA}

\author[ ]{for the Critical Path for Alzheimer's Disease\thanks{Data used in the preparation of this article were obtained from the Critical Path Institute's Critical Path for Alzheimer's Disease (CPAD) consortium. As such, the investigators within CPAD contributed to the design and implementation of the CPAD database and/or provided data, but did not participate in the analysis of the data or the writing of this report.}}
\author[ ]{the Alzheimer's Disease Neuroimaging Initiative\thanks{Data used in preparation of this article were obtained from the Alzheimer’s Disease Neuroimaging Initiative (ADNI) database (\href{url}{adni.loni.usc.edu}). As such, the investigators within the ADNI contributed to the design and implementation of ADNI and/or provided data but did not participate in analysis or writing of this report. A complete listing of ADNI investigators can be found in \href{http://adni.loni.usc.edu/wp-content/uploads/how_to_apply/ADNI_Acknowledgement_List.pdf}{this document}.}}
\author[ ]{the Alzheimer's Disease Cooperative Study\thanks{Data used in preparation of this manuscript/publication/article were obtained from the University of California, San Diego Alzheimer’s Disease Cooperative Study. Consequently, the ADCS Core Directors contributed to the design and implementation of the ADCS and/or provided data but did not participate in analysis or writing of this report.}}

\date{\today}

\begin{document}

\newcommandx{\E}[2][1]{\mathbb E_{#1} \left[#2\right]}
\newcommandx{\V}[2][1]{\mathbb V_{#1} \left[#2\right]}
\newcommandx{\C}[2][1]{\mathbb C_{#1} \left[#2\right]}\newcommand{\iid}[0]{\overset{\text{IID}}{\sim}}
\newcommandx{\Ehat}[2][1]{\widehat{\mathbb E}_{#1} \left[#2\right]}
\newcommandx{\Vhat}[2][1]{\widehat{\mathbb V}_{#1} \left[#2\right]}
\newcommandx{\hattilde}[1]{\widehat{\widetilde{#1}}}
\newcommandx{\pto}[0]{\overset{p}{\to}}

\maketitle

\begin{abstract}

Estimating causal effects from randomized experiments is central to clinical research. Reducing the statistical uncertainty in these analyses is an important objective for statisticians. Registries, prior trials, and health records constitute a growing compendium of historical data on patients under standard-of-care that may be exploitable to this end. However, most methods for historical borrowing achieve reductions in variance by sacrificing strict type-I error rate control. Here, we propose a use of historical data that exploits linear covariate adjustment to improve the efficiency of trial analyses without incurring bias. Specifically, we train a prognostic model on the historical data, then estimate the treatment effect using a linear regression while adjusting for the trial subjects' predicted outcomes (their \textit{prognostic scores}). We prove that, under certain conditions, this prognostic covariate adjustment procedure attains the minimum variance possible among a large class of estimators. When those conditions are not met, prognostic covariate adjustment is still more efficient than raw covariate adjustment and the gain in efficiency is proportional to a measure of the predictive accuracy of the prognostic model {above and beyond the linear relationship with the raw covariates}. We demonstrate the approach using simulations and a reanalysis of an Alzheimer's Disease clinical trial and observe meaningful reductions in mean-squared error and the estimated variance. Lastly, we provide a simplified formula for asymptotic variance that enables power calculations that account for these gains. Sample size reductions between 10\% and 30\% are attainable when using prognostic models that explain a clinically realistic percentage of the outcome variance.

\end{abstract}

\section{Introduction}

The goal of much clinical research is to estimate the effect of a treatment on an outcome of interest (causal inference) \cite{maldonado}. The randomized trial is the gold standard for causal inference because randomization cancels out the effects of any unobserved confounders in expectation \cite{Sox:2012hu, Overhage:2013fx, Hannan:2008gh}. Although unobserved confounding is not a concern in randomized studies, we must still contend with the statistical uncertainty inherent to finite samples if we conduct our work in a population inference framework. Because of this, methods for the analysis of trial data should be chosen to safely minimize the resulting statistical uncertainty about the causal effect. 

For a given treatment effect estimator and data-generating process, sample size is the primary determinant of sampling variance and power. Therefore the most straightforward method to reduce sampling variance is to run a larger trial that includes more subjects. However, trial costs and timelines typically scale with the number of subjects, making large trials economically and logistically challenging. Moreover, ethical considerations suggest that human subjects research should use the smallest sample sizes possible that allow for reliable decision making. 

As most clinical trials compare an active treatment to a standard-of-care, often in combination with a placebo, there is a possibility to use existing control arm data\footnote{
Treatment-arm historical data is usually more difficult to come by, especially if the active treatment under consideration is experimental and has not been previously tested. However the approach we propose is generic enough to handle relatively arbitrary historical data. The improvements in efficiency will depend on how similar the historical data-generating process is to the current one.
} to augment clinical trials and reduce variance. Such ``historical borrowing'' methods are becoming increasing attractive as the creation of large, electronic patient datasets in the past decade facilitates this process by making it easier to find a suitably matched historical population. Various approaches ranging from directly inserting subjects from previous studies  into the current sample to using previous studies to derive prior distributions for Bayesian analyses have been proposed \cite{kopp2020, 10.1002/sim.6728, lim, baker2001}. Although such methods do generally increase power, they cannot strictly control the rate of type I error \cite{ghadessi, baker2001, kopp2020}. 

Here, we describe an alternative approach that exploits machine learning models and historical control arm data to decrease the uncertainty in effects estimated from randomized trials without compromising strict type-I error rate control in the large-sample setting. The gist of our proposal is to use the historical data to train a prognostic model that predicts a patient's outcome given their baseline covariates. This prognostic model is applied to all patients in the current trial in order to generate a prediction of their outcome (their ``prognostic score'' \cite{Hansen:2008cw, aikens, Wyss:2014ef}). The score is then adjusted for as a covariate in a (linear) regression model of the trial outcome in order to estimate the treatment effect. This amounts to adding a single (constructed) adjustment covariate into an adjusted analysis. As such, it poses no additional statistical risk\footnote{
\label{foot:randomization-inference}
There is some confusion and subtlety around this point. \citet{Freedman:2008738} is sometimes referenced to claim that covariate adjustment can incur bias, but in fact, the presence of treatment interaction terms resolves the issue \cite{Lin:2013738}. Moreover, both of these results apply only to the ``randomization inference'' setting where the covariates and potential outcomes are considered fixed (i.e. randomness arises only from treatment assignment, not sampling). We use the more-common population inference framework in this work, in which these criticisms do not apply \cite{rosenblum-glm}.
}
over any other pre-specified adjusted analyses, which are preferable to unadjusted analyses in almost every case \cite{kahan, raab, yang-tsiatis, Lin:2013738}. Our approach is entirely pre-specifiable, is generic enough to be integrated into many analysis plans, and is supported by regulatory guidance \cite{ema-covar-adjustment}. Moreover, there are no practical restrictions on the type of prognostic model used, enabling one to leverage machine learning-based methods that can learn nonlinear predictive models from large databases.

Below, we show that this prognostic covariate adjustment procedure attains the minimum possible asymptotic variance among ``reasonable'' estimators as long as the prognostic model improves with more data and the effect of treatment is constant. The uncertainty in the estimate for the treatment effect is minimized when the prognostic model predicts the control potential outcome of each subject. However, one can realize gains in efficiency even with imperfect prognostic models or in the presence of heterogeneous effects. In general, our procedure decreases the variance of the estimated treatment effect in proportion to the squared correlation of the prognostic model with the outcome while guaranteeing unbiasedness, control of type-I error rate, and confidence interval coverage. We demonstrate the efficiency gain in simulations and through a reanalysis of a previously reported clinical trial studying the effect of docosahexaenoic acid (DHA) on cognition in patients with Alzheimer's Disease.

Although we are (to our knowledge) the first to formally characterize it, prognostic covariate adjustment has been used in trials for a long time. The baseline covariates in a trial are often “atomic” measurements such as sex, age, or lab values, but composite or computed covariates such as body mass index, Charlson comorbidity index, or Framingham risk score, are also frequently used \cite{cv-risk-scores, austin-comorbidity-score, Ambrosius2014, raab, kahan}. These ``indices'' or ``scores'' are usually the output of a simplified prognostic model that has been learned (at least implicitly) from historical data. For instance, the Framingham cardiovascular risk score was developed by training Cox and logistic regression models using a large community-based cohort to obtain a single covariate that is highly predictive of cardiovascular outcomes. From that perspective, our proposed approach is a formalization of what has previously been an ad-hoc procedure.

The novel contributions of this paper are therefore threefold. Firstly, we provide a formal characterization of prognostic covariate adjustment. Although this process has already been used in trials (if one considers a baseline outcome measurement or risk score as a rudimentary prognostic score), it has not been formally described as a method for leveraging historical data to improve the efficiency of a proposed trial. Our second novelty is an asymptotic proof that shows prognostic covariate adjustment is semiparametric efficient {if the effect of treatment is constant, the historical data follow the same distribution as the trial control arm, and the prognostic model improves with the amount of historical data}. Roughly translated, this means that the power of a trial using prognostic covariate adjustment will be higher or equal to the power of any other trial design that controls type I error. Thirdly, we provide a method of sample size estimation that allows for the design of smaller trials with prognostic covariate adjustment that maintain their power. The formula we derive coincides in certain special cases with previously known results, but to our knowledge our derivation is more general and rests on fewer assumptions than what is available elsewhere in the literature \cite{Borm2007-ep}.

\section{Setting and Notation}
Our setting is a two-arm randomized clinical trial. Denote the outcome for subject $i$ in the clinical trial with $Y_i$, their baseline covariates with $X_i$, and their treatment assignment with $W_i$. The trial dataset is a set of $n$ tuples $(X_i, W_i, Y_i)$, which we denote $(\bm X, \bm W, \bm Y) \in \mathcal X^n \times \{0,1\}^n \times \mathbb R^n$ (we use boldface $\bm A$ to denote a vector of random variables,  each associated with one observation in the dataset). Throughout the paper we assume a continuous outcome. Let $Y_0$ and $Y_1$ be the control and treatment potential outcomes of the subjects in the trial, respectively, and let $\bm Y_{\bm W} = \bm W \bm Y_1 + (1-\bm W) \bm Y_0$ \cite{Rubin2005}. Our structural assumption about the trial is,
%%%
\begin{equation}
 P(\bm X, \bm W, \bm Y, \bm Y_0, \bm Y_1) = 
 \bm 1(\bm Y = \bm Y_{\bm W})
 P(\bm W) 
 \prod_i P(X_i, Y_{0,i}, Y_{1,i}) \, .
 \label{rct-structure}
\end{equation}
%%%
In other words, a) the observed outcomes are the potential outcomes corresponding to the assigned treatment, b) the treatment is assigned independently of observed or unobserved baseline covariates and independently of potential outcomes, and c) our subjects are independent of each other. In addition to being independent, we also assume the subjects are identically distributed, i.e., $(X_i, Y_{0,i}, Y_{1,i}) \iid P(X, Y_{0}, Y_{1})$. 

We denote the population average outcomes under each treatment condition $w$ as $\mu_w = \E{Y_w}$ and the conditional means as $\mu_w(X) = \E{Y_w|X}$. In general, a ``treatment effect'' is any function of these marginal means, i.e. $\tau = r(\mu_0, \mu_1)$, but when we say treatment effect here we specifically mean the difference in means, $\tau = \mu_1 - \mu_0$.

Finally, we denote the treatment indicators as $W_{1,i} = W_i$ and $W_{0,i} = 1-W_i$ to allow for symmetric notation. Let $\pi_{1} = P(W_{1}=1)$ and let $\pi_0 = P(W_{0}=1)$ be the probability that a subject is assigned to the treatment or control arm in the trial, respectively. In simple randomized experiments, these are constants that apply to all subjects.

In what follows, abbreviate the usual empirical (sample) average of IID variables $A_1 \dots A_n \sim A$ with the notation $\Ehat{A} = \frac{1}{n} \sum A_i$ (or $\bar A$). Denote an empirical conditional average $\Ehat{A|B=b} = \frac{1}{n_b} \sum_{B_i=b} A_i$ with $n_b$ the number of observations where $B_i=b$. Let $\tilde A = A - \E A$ (or $\tilde A = A - \Ehat A$) be centered (or empirically centered) versions of the random variable $A$, with usage clear from context or otherwise noted. Let $\V{A}$ denote the variance of $A$ and $\C{A,B}$ denote the covariance between $A$ and $B$. 

When we describe ``asymptotic'' properties of an estimator in all cases we are referring to the asymptote where the number of observations $n$ is increasing while other properties of the data-generating process remain fixed. As usual, the ``asymptotic varaince'' $\nu^2$ of an asymptotically normal estimator $\hat\tau_n$ of a parameter $\tau$ refers to the variance of the limiting distribution of the sequence $\sqrt{n}(\hat\tau_n - \tau) \rightsquigarrow N(0,\nu^2)$. We omit the subscript $n$ where the relevant sequence is clear from context.

\section{Prognostic Covariate Adjustment}
\label{sec:approach}

At a high level, our approach to historical borrowing (which we call prognostic covariate adjustment) is to use the historical data to learn a prognostic model, then apply it to the trial to generate an additional adjustment covariate. Specifically, let $\mathcal M$ be some learning algorithm (e.g., a linear regression, random forest, deep neural network, etc.) which, when trained on our historical dataset $(\bm X', \bm Y')$, outputs a ``fit'' predictive model, $m: \mathcal X \to \mathcal Y'$. {We assume throughout that the historical and trial data are statistically independent. This is natural since we already assume that the individual observations \textit{within} each dataset are independent, but does preclude scenarios where a subject present in the historical data later enters the trial, for example.}

The treatment effect estimate we will use is consistent for any choice of prognostic model, so we can discuss the method in great generality without assuming the prognostic model takes a specific form or attains a certain level of predictive performance. That said, it is particularly interesting to consider the ``best case scenario'' in which the historical data are representative of the trial control arm, i.e. $P(X',Y') = P(X,Y_0)$. It will be later be shown that the prognostic model that endows our estimator with minimum variance is the conditional mean, $m(X) = \E{Y_0|X}$, which is the same as $m(X') = \E{Y'|X'}$ in the best case scenario. Therefore, the construction of the optimal prognostic score involves estimating a function equivalent to the conditional mean of the outcome given the covariates in the historical sample of controls. As such, constructing the optimal prognostic model is a standard machine learning problem.  In the results that follow, we make no presumption about the particular choice of $\mathcal M$. We refer to $m(\cdot)$ as a \textit{prognostic model} and to the quantity $M_i = m(X_i)$ as the \textit{prognostic score}. \citet{Hansen:2008cw} defines the prognostic score as any quantity $f(X)$ which induces $Y_0 \perp X | f(X)$. In the literature, the prognostic score is often treated more informally as the expected value of the control outcome given the baseline covariates, which motivates our terminology here.

While it is optimal to construct a prognostic model that accurately approximates $m(X) = \E{Y_0|X}$, this is not strictly necessary for using our proposed estimator. Indeed, there may be situations in which this isn't possible in practice. For example, the historical sample may be too small to reliably learn the relationship, the distribution of the covariates in the historical population may not reflect the trial population (i.e. a ``domain shift''), or the outcome in the trial may not even have been measured in the historical data (e.g., a biomarker measured with a new technology). In this latter case, we refer to $Y'$ as a ``surrogate'' outcome for $Y$. Whatever the reason, the proposed estimator will be unbiased, retain control of the type-I error rate, and decrease variance (to some degree) even if the prognostic model does not accurately approximate the control potential outcomes.

In general, we can assume that the prognostic model, $m$, is given at the time of the trial. To analyze the trial data, we first use the prognostic model to generate the prognostic score, $M_i = m(X_i)$, for each subject given their baseline covariates, $X_i$. Then, we estimate the treatment effect using a linear regression adjusted for the empirically centered covariates, prognostic score, and their interactions with the treatment\footnote{
Theorem \ref{thm:comparison} in the appendix shows that including interactions is necessary to ensure that the estimator is more efficient than difference-in-means estimation. The estimated coefficients of these interactions are of no interest to us here--they are merely a tool to reduce variance in estimation of the main treatment effect. There is a separate and large literature on the estimation of heterogeneous treatment effects, of which testing for linear treatment-covariate interactions is one small part.
}. Letting 
$
Z^\top = [1, \tilde W, \tilde X^\top, \tilde M^\top, \tilde W\tilde X^\top, \tilde W\tilde M^\top] 
$
be the regressors
\footnote{

If the prognostic score $M$ is identically equal to any of the covariates or is numerically constant, we omit including it in the regression since in this case it cannot do anything to reduce variance.
}
, we fit $\E{ Y_i \big|  Z_i} = 
Z_i^\top \beta$ using ordinary least squares to obtain the fit coefficients, $\hat \beta$. Our estimate of the treatment effect is $\hat\tau = \hat\beta_W$, i.e., the coefficient corresponding to the $W$ term in the regression. {This specification is directly based on the ``ANCOVA II'' estimator analyzed in \citet{yang-tsiatis}.} It is well-known that this is a consistent and asymptotically normal estimator of the treatment effect when treatment is randomized, even if the regression is misspecified (i.e., the true relationship is nonlinear)\footnote{See footnote \ref{foot:randomization-inference}.}
\cite{wang, leon, yang-tsiatis}.

We estimate sampling variance with the usual ``sandwich'' estimator 
$
\widehat{\V{\hat\tau}} = 
(\bm Z \bm Z^\top)^{-1} 
DD^T
(\bm Z \bm Z^\top)^{-1}
$ 
in which $D = \bm Z\operatorname{diag}(\bm Y - \bm Z^\top \hat\beta)$. This variance estimator is consistent whether or not the regressions are correctly specified and is consequently robust against deviations from linearity and homoscedasticity \cite{agnostic-stats}. This guarantees strict (large-sample)\footnote{In practice, it may be advisable to employ corrections for this estimator if one is working with a small sample. The estimators known as HC1, HC2, and HC3 are all reasonable options that make minor changes to the matrix $D$, e.g. $D_{\text{HC1}} = \sqrt{\frac{n}{n-p}}D$. See \citet{long-hc3}.} type-I error rate control based on p-value cutoffs and valid confidence intervals in practically all cases (i.e., as long as $P(Y_0,Y_1,X)$ satisfies mild regularity conditions). In particular, the statistical validity of the inference is not impacted by the nature of the prognostic score $M$ because it enters into the analysis like any other covariate.

\subsection{Statistical Properties}

{
\subsubsection{Optimality Under Constant Effects}

Although prognostic covariate adjustment provides valid inference (i.e. proper coverage and asymptotic unbiasedness) under general conditions, we first motivate it by showing that it is in fact \textit{optimal} under the assumption of a constant treatment effect and certain assumptions about the prognostic model.
}

\begin{theorem}
\label{thm:opt-main}
Presume a constant treatment effect $\mu_1(X) = \mu_0(X) + \tau$. Then the linear prognostic covariate adjustment procedure that uses $m(X) = \E{Y_0|X}$ as the prognostic model has the lowest possible asymptotic variance among all regular and asymptotically linear estimators with access to the covariates $X$. 
\end{theorem}

This is restated and proved as corollary \ref{thm:constant-effect} in the appendix. Corollary \ref{thm:constant-effect-ancovaI} shows the statement also holds when the interaction terms are omitted from the working regression model. All practical and reasonable estimators in the context of trial analyses with continuous outcomes are regular and asymptotically linear,\footnote{
Regularity and asymptotic linearity are specific technical conditions. Definitions may be found in \citet{Tsiatis:2007vl}.
} so this result means linear adjustment for the true control conditional mean is, in some sense, the absolute ``best'' possible estimation procedure \cite{Tsiatis:2007vl}. Since the prognostic model is trained to approximate $\E{Y'|X'}$, it is best for the historical data to be drawn from the same distribution as the trial control arm so that $\E{Y'|X'} = \E{Y_0|X}$. 

We can in fact weaken our assumption that the prognostic score is perfect ($m(X) = \E{Y_0|X}$) and still obtain the same conclusion as long as the prognostic model approaches the truth as the size of the external dataset increases. The proof of the following is also given in the appendix (theorem \ref{thm:asymptotic}).

\begin{theorem}
Presume $X$ has compact support and there is a constant treatment effect: $\mu_1(X) = \mu_0(X) + \tau$ with $|\mu_0(x)| < b$ bounded. Let $m(x)$ be a (random) function learned from the external data $(\bm Y', \bm X')_{n'}$ such that $|m(x)| < b$ is also bounded and {$|m(X) - \mu_0(X)| \overset{L_2}{\to} 0$ so that the learned model approaches the truth in mean-squared error} as $n' \to \infty$. If the number of trial samples $n$ grows in tandem with the size of the historical data $n'$ (i.e. $n = O(n')$), then the linear prognostic covariate adjustment estimator that uses the learned model $m(X)$ in the role of $X$ has the lowest possible asymptotic variance among all regular and asymptotically linear estimators with access to the covariates $X$. 
\end{theorem}

This also holds without the use of interaction terms in the regression (corollary $\ref{thm:asymptotic-I}$). Of course, $m(X) = \E{Y_0|X}$ will never hold for any given prognostic model, either because the model is not perfectly learning the relationship or because the training data are not entirely representative of the trial control arm. Despite this, the theorem means that our statement of optimality is not meaningless because with enough external data we will come closer and closer to attaining the minimum possible variance. Note that the condition that {$|m(X) - \mu_0(X)| \overset{L_2}{\to} 0$} is relatively weak and justifies the use of a number of machine learning algorithms $\mathcal M$ to learn the prognostic model \cite{convergence-L2boost, convergence-lasso, convergence-nn, convergence-rf}.

{
\subsubsection{Superiority Over Covariate Adjustment Without Prognostic Score}

It is not always reasonable to assume a constant effect. However, we can show that even if it is not optimal among all estimators, prognostic covariate adjustment still retains an advantage over standard covariate adjustment (def. \ref{def:diff}). This holds for any fixed prognostic model $m(x)$.

\paragraph{A Note on Interaction Terms} The following results require the aforementioned interaction terms to be present in the regression model. If interactions are omitted, the results below only hold if either the treatment effect is constant or the randomization is 1:1. Without these conditions and without interaction terms, there are cases in which adding the prognostic score (or any covariate) could actually increase the asymptotic variance (thm \ref{thm:comparison}). These cases may be rare in practice but either way the problem can be entirely avoided simply by including interactions. Theorem \ref{thm:comparison} also shows that these interactions are always useful even without the prognostic score.
}

\begin{theorem}
\label{thm:improvement}
Assume only mild regularity conditions to ensure that the usual sandwich estimator is consistent. Consider linearly adjusting (with treatment-covariate interactions) for covariates $X$ with variance $\V{X} = \Sigma_x$ and covariance with $Y_w$ of $\C{Y_w, X} = \xi_{w,x}$ vs. a set of covariates $[X, M]$ ($M \in \mathbb R$) {(again using interactions)} with $\C{X,M} = \zeta$, $\V{M} = \sigma_m^2$ and $\C{Y_w, M} = \xi_{w,m}$. Let $\xi_{m*} = \pi_0 \xi_{1,m} + \pi_1 \xi_{0,m}$ and $\xi_{x*} = \pi_0 \xi_{1,x} + \pi_1 \xi_{0,x}$. { Assume $M$ is not a linear combination of the variables in $X$.} The reduction in asymptotic variance from including the prognostic score $M$ in the regression is always {nonnegative} and given by

\begin{align}
    \left(
    \frac{1}{\pi_0 \pi_1} 
    \right)
    \frac{
        (\xi_{m*} - \xi_{x*}^\top \Sigma_x^{-1} \zeta)^2
    }{
        \sigma_m^2 - \zeta^\top \Sigma_x^{-1} \zeta
    }
\end{align}
\end{theorem}

This is restated and proved as corollary \ref{thm:add-covar} in the appendix. The conclusions are that (i) including a prognostic score as a covariate should never hurt the asymptotic variance and (ii) the efficiency gain depends on how well correlated the prognostic score is with the outcome {(above and beyond any correlation with the raw covariates)}. Prognostic models that are better correlated with the outcome thus offer larger efficiency gains, but the presence of any correlation at all could still decrease the variance. This justifies the use of prognostic covariate adjustment for surrogate outcomes (i.e. when $Y'$ and $Y$ represent different but, perhaps correlated, outcomes) and in cases with heterogeneous treatment effects. { It is worth mentioning that the prognostic score must be a nonlinear function of the included covariates $X$ for there to be any benefit. In fact, there would not be a unique solution for the coefficients of the regression if the prognostic score and covariates were exactly colinear.}

{
\subsubsection{Pragmatic Sample Size Calculation}
}

Our primary goal in this paper is to demonstrate how prognostic covariate adjustment  decreases the variance for a fixed trial relative to raw covariate adjustment. However, it is also useful to consider how knowledge of this efficiency gain could be exploited to conduct \textit{smaller} trials that attain a desired level of confidence. Since the asymptotic variance of the estimate determines the power of the trial, smaller trials using prognostic covariate adjustment may attain equal power to a larger trial using raw covariate adjustment. Power must be estimated \textit{before} running the trial in order to gauge the number of subjects to be enrolled. This requires a formula for sampling variance. Although theorem \ref{thm:improvement} gives a precise quantification of the efficiency gained by including the prognostic score, it may be impractical to calculate this quantity because one would need to estimate a potentially large number of population parameters (e.g., the entries of $\Sigma_x$). Here we provide an upper bound on the resulting sampling variance that uses fewer population parameters and is therefore easier to apply in a prospective setting.

\begin{theorem}
\label{thm:pwr}
Assume only mild regularity conditions to ensure that the usual sandwich estimator is consistent. Given an arbitrary, fixed prognostic score $M=m(X)$, the asymptotic variance of our proposed estimation procedure is no greater than 
\begin{equation}
\label{eq:var-procova}
  \frac{\sigma_0^2}{\pi_0} 
+ \frac{\sigma_1^2}{\pi_1}
- \pi_0 \pi_1
\left( 
      \frac{\rho_1\sigma_1}{\pi_1} 
    + \frac{\rho_0 \sigma_0}{\pi_0} 
\right)^2
\end{equation}
where $\sigma_w = \V{Y_w}$ and $\rho_w$ is the population correlation between $M$ and $Y_w$, i.e. $\C{M, Y_w}/\sqrt{\V{M} \V{Y_w}}$. This is always less than the asymptotic variance of the unadjusted estimator (an upper bound on the variance of the standard covariate-adjusted estimator that uses the same population parameters), which is
\begin{equation}
\label{eq:var-anova}
  \frac{\sigma_0^2}{\pi_0} 
+ \frac{\sigma_1^2}{\pi_1}
\end{equation}
\end{theorem}

The bound follows directly from \ref{thm:ancovaii-var}, \ref{thm:comparison}, and \ref{thm:add-covar} in the appendix. This bound does not account for any reduction in variance due to adjustment for the raw covariates, $X$. However, if the prognostic model is accurate, the raw covariates are unlikely to provide substantial efficiency gains because their effects are already ``soaked up'' by the prognostic score. In addition, the only population parameters in this bound are the marginal outcome variances and model-outcome correlations in each treatment arm.

In conjunction with estimates of $\sigma^2_w$ and $\rho_w$, this formula may be used to prospectively calculate a lower bound on the power of a clinical trial analyzed with prognostic covariate adjustment. Asymptotic normality means that, in the limit, the probability of a two-sided $p$-value being less than $\alpha$ (i.e. a ``statistically significant'' result) is

\begin{equation}
\label{eq:norm-pwr}
\Phi
    \left(
        \Phi^{-1}(\alpha/2) + 
        \sqrt{n}\frac{\tau}{\nu}
    \right) 
+
\Phi
    \left(
        \Phi^{-1}(\alpha/2) -
        \sqrt{n}\frac{\tau}{\nu}
\right)
\end{equation}
where $\Phi$ is the CDF of the standard normal, $\tau$ is the true (target) treatment effect and $\nu^2$ is the asymptotic variance of whatever asymptotically normal estimator is being used. Composing the variance bound for the prognostic covariate adjustment estimator given in theorem \ref{thm:pwr} with this formula gives an upper bound for the power of trial analyzed with that estimator that depends only on the target effect $\tau$, sample sizes $n_w$, potential outcome variances $\sigma^2_w$, and potential outcome-prognostic score correlations $\rho_w$. The target effect is usually fixed a-priori. The latter two quantities may be estimated using historical data and/or expert opinion (see appendix \ref{appx:pwr}). After fixing them, the sample sizes may be varied in a desired ratio until the desired power (e.g. 80\%) is achieved.

Because of efficiency gains, the required sample size for a trial powered with this method and analyzed with prognostic covariate adjustment will be lower than for a trial powered without exploiting covariate information. We can equate the power for a target effect $\tau$ of an estimator with variance 
given by eq. \ref{eq:var-procova} (upper bound on prognostic covariate adjustment)
to one with variance given by \ref{eq:var-anova} (upper bound on raw covariate adjustment) to algebraically discover the relationship between $n$, the sample size required for a well-powered analysis with raw covariate adjustment, and $n^\dagger$, the sample size required for a well-powered analysis with prognostic covariate adjustment. In the case of a 1:1 randomization ratio ($\pi_0 = \pi_1 = 0.5$), this relationship is

\begin{equation}
\label{eq:1to1-n}
\frac{n^\dagger}{n}
    =
    1 - \frac{(\sigma_0 \rho_0 + \sigma_1 \rho_1)^2}{2(\sigma_0^2 + \sigma_1^2)} 
\end{equation}

With a common variance $\sigma_0 = \sigma_1 =\sigma$ and correlation $\rho_0 = \rho_1 = \rho$ this simplifies to
$
\frac{n^\dagger}{n}
    =
    1 - \rho^2 
$, or, in terms of a percent reduction in sample size from $n$, $
\frac{n-n^\dagger}{n} = \rho^2
$. This is precisely the out-of-sample $R^2$ of the prognostic model $m$. This relationship holds regardless of the value of the target effect and the desired power. The result coincides with the formula from \citet{Borm2007-ep}, although our derivation is more robust because we do not rely on parametric assumptions. A generic relationship that does not assume 1:1 randomization is easily derived in the same fashion.

$R^2$ values can vary wildly depending on the model, outcome, and population, but folk wisdom among biostatisticians has it that values between 0.1 and 0.3 might be expected from a good prognostic model. Those values translate to meaningful sample size reductions between 10\% and 30\%.

Note that the bound in theorem \ref{thm:pwr} is actually generic to any covariate because $M$ enters the estimator the same way any other covariate does. In other words, the result still holds if one substitutes any $X_j$ for $M$ in the theorem. For instance, presume baseline age $X_0$ is known to be correlated with the standard-of-care (control) outcome at a strength of $\rho_{X_0} = 0.2$. Then in a 1:1 trial (presuming common $\sigma^2$ and $\rho$) we could reduce the sample size by 4\% relative to the unadjusted power calculation and still maintain the same design power for the same target effect, all without the need for a dedicated prognostic model\footnote{
A prognostic model will, of course, be expected to approach the maximum possible correlation with the outcome that could be attained as a function of the baseline covariates. Thus there is usually a benefit to using such a model if one is available.
}.

\section{Simulations}

Unlike analyses of real data, analyses of simulated data can be compared to ground truth to gauge error. We used simulation to explore how mean-squared estimation error of the treatment effect varies with and without prognostic covariate adjustment. In particular, we were interested in cases with or without strong effects, non-linearity in the outcome-covariate relationship, heterogeneity in the treatment effect, surrogate outcomes, or distributional shifts between the historical and trial data-generating processes. We show that prognostic covariate adjustment performs better than raw covariate adjustment in all cases where theory expects it to and performs worse in no cases. The amount of improvement from scenario to scenario is in line with intuition gained from theory.

We chose our simulation scenarios to demonstrate the different performance benefits of effect estimation with prognostic covariate adjustment relative to raw covariate adjustment in cases that might come up in real trials. Effect sizes vary from trial to trial but the overall size of the true treatment effect should not change the relative efficiency of prognostic and raw covariate adjustment (i.e. because the asymptotic variances do not depend on the treatment effect) so we chose this scenario to demonstrate the point. Possible nonlinearities must always be contended with in trial analyses. For our purposes, outcomes that are linearly related to the covariates should not benefit from prognostic covariate adjustment above and beyond raw covariate adjustment because all of the information from the raw covariates is exploitable by the linear model. Heterogeneity of effect is always possible and should decrease the advantage of prognostic covariate adjustment to some extent because the prognostic score is less predictive in the treatment arm. Using a surrogate outcome should also decrease the advantage of the prognostic score because its predictive capacity is lessened in both treatment arms of the trial-- this case may arise if the trial outcome is not well-measured in historical data. Lastly, it is always possible that the historical data are not representative of the current trial population. When this is the case, we expect the benefits of prognostic covariate adjustment to be attenuated because the model must extrapolate outside of its trained range to perform on the trial population.

Each of our simulation scenarios is defined by particular choices for the pair of distributions $P(X',Y')$ and $P(X,Y_0, Y_1)$. In all cases, the distribution of covariates in the simulated historical and trial data were 10-dimensional uniform random variables in the prism $[l,h]^{10}$. Distributional shift was modeled by choosing different values of $l$ and $h$ for $P(X')$ and $P(X)$. The distributions $P(Y'|X')$, $P(Y_0|X)$, and $P(Y_1|X)$ were of a Gaussian quadratic-mean form $\mathcal N(aX^\top \mathds 1 X + bX^\top \mathds 1 + c, 1)$ in all scenarios ($\mathds 1$ is a matrix or vector of 1s with appropriate shape implied). The parameter $a$ controls the degree of non-linearity, with $a=0$ representing the linear case. In this context, treatment effect heterogeneity refers to the situation in which $a$ or $b$ is different for $P(Y_0|X)$, and $P(Y_1|X)$, whereas surrogate outcome refers to the situation in which $a$ or $b$ is different for $P(Y'|X')$ and $P(Y_0|X)$. Large constant effects are encoded with different values for $c$ in $P(Y_0|X)$, and $P(Y_1|X)$ while keeping $a$ and $b$ the same. The specific values of $l$, $h$ for each covariate distribution and of $a$, $b$, and $c$ are shown in Table \ref{tbl:sim_scenarios}.

\begin{table}[p]
\centering
\begin{tabular}{|c||c|c|c|c|c|c|c|c|c|c|c|c|c|}
\hline
\multirow{2}{*}{Scenario} &
\multicolumn{2}{c|}{$P(X')$} &
\multicolumn{2}{c|}{$P(X)$} &
\multicolumn{3}{c|}{$P(Y'|X')$} &
\multicolumn{3}{c|}{$P(Y_0|X)$} &
\multicolumn{3}{c|}{$P(Y_1|X)$}
\\ \cline{2-14}
 &
$l'$ & $h'$ &
$l$ & $h$ &
$a'$ & $b'$ & $c'$ &
$a_0$ & $b_0$ & $c_0$ &
$a_1$ & $b_1$ & $c_1$
\\ \hline \hline
\textbf{Baseline} &
\textbf{-1} & \textbf 1 &
\textbf{-1} & \textbf 1 &
\textbf {0.5} & \textbf 1 & \textbf 0 &
\textbf{0.5} & \textbf 1 & \textbf 0 &
\textbf{0.5} & \textbf 1 & \textbf 0
\\ 
Strong Effect &
-1 & 1 &
-1 & 1 &
0.5 & 1 & 0 &
0.5 & 1 & 0 &
0.5 & 1 & \textcolor{purple}5
\\
Linear &
-1 & 1 &
-1 & 1 &
\textcolor{purple}0 & 1 & 0 &
\textcolor{purple}0 & 1 & 0 &
\textcolor{purple}0 & 1 & 0
\\ 
Heterogeneous Effect &
-1 & 1 &
-1 & 1 &
0.5 & 1 & 0 &
0.5 & 1 & 0 &
\textcolor{purple}0 & 1 & 0
\\ 
Surrogate Outcome &
-1 & 1 &
-1 & 1 &
0.5 & \textcolor{purple}{-1} & 0 &
0.5 & 1 & 0 &
0.5 & 1 & 0
\\ 
Covariate Shift &
\textcolor{purple}{-2} & \textcolor{purple}0 &
-1 & 1 &
0.5 & 1 & 0 &
0.5 & 1 & 0 &
0.5 & 1 & 0
\\ \hline
\end{tabular}
\caption{Parameters for all simulation scenarios. Parameters for the baseline scenario are shown in bold. Parameters in the other scenarios that deviate from the baseline parameters are highlighted in purple.}
\label{tbl:sim_scenarios}
\end{table}

The ``baseline'' simulation scenario included some moderate outcome non-linearity, constant treatment effect, and no distributional shift between the historical data and the trial control arm. We tested four variations of this scenario. In the first (``linear'') we examined what happened when the outcome-covariate relationship was precisely linear in both treatment arms. In the second (``heterogeneous effect'') we tested a variation of the baseline scenario where the conditional average effect $\E{Y_1 - Y_0 | X}$ was no longer a constant. In the third (``surrogate outcome'') we tested a variation where the relationship between outcome and covariates in the historical data was not representative of the corresponding relationship in the control arm of the trial, i.e. $P(Y'|X'=x) \ne P(Y_0|X_0=x)$. In the fourth (``covariate shift'') we tested a variation where the historical population was not representative of the trial population in terms of the baseline covariates, i.e. $P(X'=x) \ne P(X=x)$. 

In each simulation scenario, we generated a historical control dataset $(\bm X', \bm Y')$ by drawing 10,000 IID samples from a specified distribution, $P(X', Y') = P(Y'|X')P(X')$. These simulated historical data were used to train to a random forest (1000 trees, with other parameters set to defaults in the python package sklearn \cite{sklearn}) as a prognostic model, $m: \mathcal X \to \mathcal Y$. Then, we simulated a randomized trial dataset $(\bm X, \bm W, \bm Y)$ with 500 subjects, evenly split between treatment and control. The data-generating process for these data involved drawing 500 IID samples from a counterfactual distribution $P(Y_1, Y_0, X) = P(Y_1|X)P(Y_0|X)P(X)$, evenly splitting the sample into treatment and control arms, and then setting $Y=Y_1$ for the treated and $Y=Y_0$ for the controls. Finally, we used the prognostic model to generate the prognostic score, $\bm M = m(\bm X)$, and  analyzed the data using four estimation procedures: unadjusted, covariate-adjusted, covariate-adjusted with prognostic score, and covariate-adjusted with prognostic score sans interactions.\footnote{
``Adjusted for'' means by default that we included both the main effect of the covariate and its interaction with the treatment. We also report results for an estimator adjusted for the covariates and prognostic score, but excluding the interaction term. Regulatory guidance often recommends against the inclusion of interactions without strong a-priori evidence for their existence \cite{ema-covar-adjustment}. And although theoretically useful, the practical benefits from including the interactions may be negligible in practice. Indeed, it follows from theorem \ref{thm:comparison} that the benefit disappears when there is no heterogeneity of effect.
}. We also repeated the prognostic-score-and-covariate-adjusted analysis using the true ``oracle'' prognostic score $\E{Y_0|X=x}$ to show a best case scenario in terms of the prognostic model (this estimator would not be feasible in practice). Results from additional regression specifications are available in the appendix.

The result was a set of five effect estimates. We calculated the squared-error of each estimate relative to the true treatment effect, $\E{Y_1-Y_0}$, known from the data-generating counterfactual distribution, repeated this process 10,000 times, and averaged the squared-errors to obtain mean-squared errors for each estimator. The results are shown in Table \ref{tbl:sim_results}.

\begin{table}[p]
\centering
\begin{tabular}{|c||c|c|c|c|c|}
\hline
Scenario & 
unadjusted & 
\makecell{+covariates\\ +interactions} &
\makecell{+covariates\\ +prognostic score \\ +interactions} &
\makecell{+covariates\\ +prognostic score} & 
oracle \\ 
\hline
Baseline &
0.076 &
0.051 &
\textbf{0.017} &
\textbf{0.017} &
0.008 
\\
Strong Effect &
0.077 &
0.051 &
\textbf{0.018} &
\textbf{0.018} &
0.008 
\\
Linear &
0.035 &
\textbf{0.008} &
\textbf{0.008} &
\textbf{0.008} &
0.008 
\\
Heterogeneous Effect&
0.055 &
0.030 &
\textbf{0.021} &
0.022 &
0.020
\\
Surrogate Outcome&
0.075 & 
0.050 &
0.038 &
\textbf{0.037} &
0.008 
\\
Covariate Shift &
0.077 &
0.050 &
\textbf{0.049} &
\textbf{0.049} &
0.008 
\\
\hline
\end{tabular}
\caption{Mean-squared errors (MSEs) of each estimator in each simulation scenario. The result with the smallest MSE (excluding the oracle estimator) is shown in bold. The headings correspond to: 
$
Z^\top = [1, \tilde W] 
$ (unadjusted),
$
Z^\top = [1, \tilde W, \tilde X^\top, \tilde W\tilde X^\top] 
$
(+covariates +interactions), 
$
Z^\top = [1, \tilde W, \tilde X^\top, \tilde M^\top, \tilde W\tilde X^\top, \tilde W\tilde M^\top] 
$
(+covariates +prognostic score +interactions), 
$
Z^\top = [1, \tilde W, \tilde X^\top, \tilde M^\top] 
$
(+covariates +prognostic score), 
$
Z^\top = [1, \tilde W, \tilde W\tilde \mu_0(X)^\top] 
$
(oracle)
}
\label{tbl:sim_results}
\end{table}

The simulation results support the theory in that the mean-squared error of the estimator with prognostic covariate adjustment was always less than or equal to the mean-squared error without it. The only two cases in which prognostic covariate adjustment did not substantially decrease the mean-squared error were the linear and distributional shift scenarios. In the former, the linear prognostic relationship is already captured by the other baseline covariates so there is no additional benefit to adding the prognostic score, as expected. In the latter, the prognostic model may not generalize well to the study population, thereby losing a significant portion of its predictive power. All other scenarios, however, demonstrate that adding the prognostic score significantly reduced the mean-squared error. Finally, most or all of the benefit was realized without including the interaction terms, except for a modest gain in the scenario with treatment effect heterogeneity.

\section{Case Study}

In addition to the simulations presented above, we re-analyzed data from an existing trial to demonstrate how prognostic covariate adjustment decreases variance relative to a standard covariate-adjusted analysis. Our results show that prognostic covariate adjustment decreases the estimated standard errors relative to raw covariate adjustment.

Our demonstration trial, reported by \citet{dha}, was conducted to determine if docosahexaenoic acid (DHA) supplementation slows cognitive and functional decline for individuals with mild to moderate Alzheimer's disease. The trial was performed through the Alzheimer's Disease Cooperative Study (ADCS), a consortium of academic medical centers and private Alzheimer disease clinics funded by the National Institute on Aging to conduct clinical trials on Alzheimer disease. 

Quinn et al. randomized 238 subjects to a treatment arm given DHA and 164 subjects to a control arm given placebo. This trial measured a number of covariates at baseline including demographics and patient characteristics (e.g. sex, age, region, weight), lab tests (e.g. blood pressure, ApoE4 status \cite{apoe4}), and component scores of cognitive tests. A full list of the 37 covariates we used is available in the appendix. Any missing covariate values were mean-imputed in our reanalysis.

The primary outcome of interest for our reanalysis was the increase in the Alzheimer's Disease Assessment Scale - Cognitive Subscale (ADAS-Cog 11, a quantitative measure of cognitive ability) \cite{adas-cog} over the duration of the trial (18 months). Decrease in an Activities of Daily Living (ADL) score \cite{adl-ad} and increase in Clinical Dementia Rating (CDR) \cite{cdr-ad} were also recorded in the study and we treated these as secondary endpoints to demonstrate the benefits of prognostic covariate adjustment for a surrogate outcome. 

Before examining the trial data, we fit a prognostic model for the increase in ADAS-Cog 11 over 18 months conditional on the measured covariates. To train the model we used a large historical training dataset comprised of 6,919 early-stage Alzheimer's patients. These data came from the Alzheimer's Disease Neuroimaging Initiative (ADNI) and the Critical Path for Alzheimer's Disease (CPAD) database \cite{cpad1,cpad2}, and included measurements of ADAS-Cog 11 at 6-month, or more frequent, intervals post-baseline. The ADNI dataset is made up of longitudinal data from 4 sequential large observational studies in Alzheimer's disease, while the CPAD dataset is made up of control arm data from 29 Alzheimer's disease clinical trials. These data also included the same baseline covariates as were measured in the DHA trial (imputed to a column mean where missing). We used a random forest with 1000 trees to learn the prognostic model; all other parameters were left to their defaults in the python sklearn package \cite{sklearn}.

Once fit, we applied our prognostic model to generate a prognostic score for each subject in the trial dataset; that is, we used the trained random forest model to predict the change in ADAS-Cog11 for a particular patient under standard-of-care. We then estimated the treatment effect on each outcome using three different methods: (i) difference-in-means (i.e. unadjusted linear regression), (ii) linear regression adjusted for the baseline covariates, and (iii) linear regression adjusted for the baseline covariates and the ADAS-Cog 11 prognostic score corresponding to the appropriate timepoint for the trial, and (iv) the same estimator as case (iii) but with the interaction terms omitted. Note that in cases (iii) and (iv) the ADL and NPI outcomes were analyzed with our ADAS-Cog 11 prognostic score, not a separate ADL or CDR prognostic score. The purpose of this is to demonstrate how prognostic covariate adjustment works for a surrogate outcome. We report results in terms of an estimated effect and 1.96 times an estimated standard error (i.e. a 95\% confidence interval) in Table \ref{tbl:dha}.

\begin{table}[p]
\centering
\begin{tabular}{ |c||c|c|c|c| } 
\hline
Outcome & unadjusted & \makecell{+covariates\\ +interactions} &  \makecell{+covariates\\ +prognostic score \\ +interactions} &
\makecell{+covariates\\ +prognostic score}\\ 
\hline
ADAS-Cog 11 & 
-0.10 ±2.02 & 
0.58 ±1.72 & 
\textbf{0.56 ±1.69} &
\textbf{0.41 ±1.69} \\
ADL & 
-0.31 ±3.11 &
0.27 ±2.58 &
0.34 ±2.56 &
\textbf{-0.06 ±2.43} \\
CDR & 
-0.01 ±0.65 &
\textbf{0.03 ±0.54} &
\textbf{0.03 ±0.54} &
\textbf{0.04 ±0.54} \\
\hline
\end{tabular}
\caption{Results from the reanalysis of the DHA trial. Results are shown in terms of estimated effect $\pm$ 1.96 $\times$ estimated standard deviation (sandwich estimator). The result with the smallest estimated standard deviation is in bold for each outcome (row). Headings are the same as in table \ref{tbl:sim_results}.}
\label{tbl:dha}
\end{table}

Concordant with our simulation studies, the standard errors for the effect obtained using prognostic covariate adjustment were less than or equal to those obtained using unadjusted for standard covariate-adjusted analyses. This led to narrower confidence intervals (which are still theoretically guaranteed to have the correct frequentist coverage). Using prognostic covariate adjustment decreased the standard error for estimated treatment effects on ADAS-Cog11 and ADL, even though we adjusted for the predicted change in ADAS-Cog11 for both outcomes. There was negligible impact of including the surrogate prognostic score for CDR. Similarly, including the interaction term in the prognostic-adjusted regressions had little, or deleterious, impact in this experiment.

Covariate and prognostic covariate adjustment did modify the point estimates for the treatment effects to some extent, but only minimally relative to the size of the estimated standard errors. Even though the point estimates for individual studies can change, adjusting for baseline covariates or a prognostic score does not add bias \cite{wang, leon, yang-tsiatis}. In this particular trial, none of the outcomes demonstrated statistically significant improvements regardless of the analysis used, consistent with the original analyses of these data \cite{dha}.

\subsection{Power Analysis}

In addition to the reanalysis, we considered how we would power this trial had we been involved prospectively. Since trials must be designed before they are run, we did not use any data collected from the DHA trial to do this. 

Based on our historical training dataset, we estimated the marginal outcome variance to be $\sigma^2_0 = 61.76$ and the cross-validated prognostic model-outcome correlation to be $\rho_0 = 0.44$ (the maximum correlation between any single covariate and the outcome in the training data was 0.39, for baseline ADAS Cog 11 comprehension sub-score). Assuming common variances $\sigma^2_0 = \sigma^2_1$ the anticipated upper bound of the sampling variance of the unadjusted estimator (eq. \ref{eq:var-anova}) came out to 0.64 with 238 treated and 164 control subjects ($n = 402$, 3:2), giving a power of $\ge80\%$ to detect an effect size of 2.25 points in ADAS-Cog11 via eq. \ref{eq:norm-pwr}.

Assuming these same parameters and a common correlation $\rho_0 = \rho_1$, numerical optimization\footnote{
Note that eq. \ref{eq:1to1-n} does not apply here because the desired randomization ratio is not 1:1. Despite this, the obtained sample size reduction is still close to $\rho^2$.
} of eq. \ref{eq:var-procova} composed with eq. \ref{eq:norm-pwr} at the same 3:2 randomization ratio showed that $\ge80\%$ power was attainable with only 131 control and 190 treated subjects ($n=321$) when knowledge of the prognostic model was leveraged in the design. This represents an impressive 20\% reduction in the enrollment that would have been required. 

\section{Discussion}

Our theoretical and empirical results demonstrate that linear adjustment for a prognostic score is an effective and safe method for leveraging historical data to reduce uncertainty in randomized trials. In fact, prognostic covariate adjustment is optimal (i.e., semiparametric efficient) when there is a constant treatment effect and the prognostic model accurately predicts the conditional mean of the trial control arm. These benefits may also be exploited to design smaller trials that maintain their power.

Heuristically, the reason that prognostic covariate adjustment via a linear model improves efficiency is that the prognostic score captures nonlinear relationships between the covariates and outcome that the linear model could not otherwise exploit. This helps to ``explain away'' some amount of previously unexplained variability in the outcome. It is easy to show that if the prognostic score is a linear function of the coefficients (i.e., $m(X) = X^\top\alpha$) then adjusting for it in addition to the covariates cannot improve efficiency (in fact, the linear model becomes indeterminate). The prognostic model must therefore be nonlinear to provide any benefit. Our simulation results bear out this conclusion; there was no gain from prognostic covariate adjustment in the linear scenario. The lack of improvement in the covariate shift scenario may also be due to a similar phenomenon--the smaller the support of the covariate distribution, the better a linear (or constant) approximation will capture the outcome-covariate relationship.

Together, these results can be summarized as follows: adjusting for a prognostic score obtained from a nonlinear predictive model trained on a large database of historical control arm data provides near optimal treatment effect estimates in randomized trials with continuous outcomes (given previously stated conditions). There has recently been tremendous growth in the availability and performance of technology for nonlinear regression modeling (i.e., supervised machine learning), particularly in the area of deep learning. The intersection of this technological development with the creation of large historical control databases provides an opportunity to use prognostic covariate adjustment to substantially improve future clinical trials. 

As an alternative, one might forego the historical data altogether and use a nonlinear analysis directly on the trial data in order to account for any nonlinearities in the relationship between covariates and outcome. Several semiparametric efficient procedures exist for this kind of nonlinear/adaptive covariate adjustment \cite{Chernozhukov:2018fb, Wager:2016dz, rothe-rct}. If there is effect heterogeneity, adaptive adjustment methods may have the upper hand because the prognostic score will be less predictive of the treatment-arm outcomes. One practical advantage of adaptive adjustment is that there is no need for historical data at all, but that also implies that no information from historical data sources is used to improve the estimate (e.g., it may be easier to capture nonlinearities using large historical datasets, which could include hundreds of thousands of samples, than it is in small trial datasets).

On the other hand, linear prognostic covariate adjustment has its own set of practical advantages. Perhaps the most important of these is that the analysis is a standard linear regression once the prognostic score has been calculated. This makes it easy to explain and interpret for trialists, simple to implement with existing software, and makes the approach suitable under current regulatory guidance \cite{ema-covar-adjustment}. The estimator can be further simplified by omitting the interaction terms if little effect heterogeneity is suspected. The procedure is also modular: construction of the prognostic model may be outsourced to a group of machine learning experts, which also makes it possible to separate access to the historical and trial datasets. In fact, the historical data can be used to train a prognostic model within a privacy preserving framework with guaranteed protection of private health information \cite{dankar2012application,brisimi2018federated}. Lastly, our conservative variance bound (theorem \ref{thm:pwr}) makes it easy to prospectively power a trial without estimating or assuming a large number of population parameters.

It is also possible that prognostic covariate adjustment retains a statistical advantage in finite samples relative to direct nonlinear adjustment, though we have left theoretical investigation of this question to future studies. The flexibility of our procedure with respect to the prognostic model also allows for the exploitation of large proxy datasets where a surrogate outcome was measured. {This includes cases where the causal contrast of interest is between two non-baseline treatments (e.g. a trial with three or more treatment levels).} The use of historical data could therefore serve as an effective regularizer for learning in small samples, but more research will be necessary to make that conclusion. It should also be possible to combine the advantages of multiple procedures, i.e., to perform adaptive adjustment for a fixed prognostic model trained on historical data. 

Further, we show that efficiency is usually improved to some extent even if the prognostic model is imperfect or the treatment effect is not constant. In fact, prognostic covariate adjustment can never hurt asymptotic efficiency. For some studies, including a prognostic score as an adjustment covariate could mean the difference between a null result and a clear demonstration of efficacy or harm. Moreover, the unbiasedness, type-I error rate control, and correct confidence interval coverage of prognostic covariate adjustment are inherent properties of the procedure for any choice of prognostic model that only uses information from a subject's baseline covariates. 

For our efficiency improvement to hold in practice, all that is necessary is for the trial to be large enough for the asymptotic variance to be a reasonable estimate and for mild technical regularity assumptions (which underpin most asymptotic theory) to hold.\footnote{
These are technical conditions that might allow for the exchange of integration and differentiation or guarantee the existence of a mean value. They have no practical importance or violation in almost all real-world scenarios.
} These assumptions are not unique to prognostic covariate adjustment and are required for almost any meaningful frequentist analysis of the data.

It should also be possible to exploit prognostic covariate adjustment as a component in other kinds of estimators (repeated measures, binary outcomes, survival models etc.). We have limited our theoretical discussion here to the linear model since it is so common, but a prognostic score may be used as a covariate in any analysis that allows for covariate adjustment. It remains to be seen what optimality properties are satisfied by doing prognostic covariate adjustment in each kind of analysis and under what conditions.

prognostic covariate adjustment also presents a method for optimally reducing the dimensionality of the adjustment covariate set even when performing standard linear analyses. In our empirical demonstration we included 37 covariates (and their interactions with treatment) in the linear regression model in order to present a fair comparison with prognostic covariate adjustment, but including this many adjustment covariates is rarely if ever done in practice. The prognostic score offers an opportunity to exploit the information present in all of those covariates without necessarily having to include each of them in the analysis. However, including the raw covariates by themselves helps ensure some reduction in variance even if the prognostic score has a very poor correlation with the outcome. Similarly, including treatment-covariate interaction terms can only be beneficial (in large samples) but there may be a practical limit to the number of terms an analyst is willing to include in the analysis.

Regardless of theoretical considerations, our work shows that prognostic covariate adjustment offers practical advantages even when the assumptions that guarantee some forms of optimality are violated. This is borne out by the analyses of the secondary ADL and and CDR outcomes using the ADAS-Cog 11 prognostic score in our empirical demonstration. Since there is no (asymptotic) harm in including a prognostic score, it may behoove trialists working with large-enough samples to amass and exploit a variety of prognostic scores for each analysis.

We used random forests as the prognostic models in our demonstration, but alternative methods may be preferred in practice. Missing covariates, multiple longitudinal outcomes, and high-dimensional covariates (e.g. a whole genome) may be present in real trial data. Deep learning methods (and in particular generative deep learning methods) are often well suited to handle these challenges \cite{walsh-ad-model, rajkomar, lecunn-dl, miotto-dl-healthcare}. Deep learning methods can also exploit transfer learning to improve performance when the relevant historical data are meager \cite{dubois-transfer}. Moreover, most trials actually have many outcomes of interest (e.g., multiple primary or secondary endpoints, adverse events, and biomarkers monitored for safety). Often, each of these outcomes is measured at multiple timepoints during a trial. In principle, the development of comprehensive, longitudinal predictive models of patient outcomes under standard-of-care would enable prognostic covariate adjustment to be used for each of these analyses, thereby enabling the design of studies that require fewer subjects to achieve desired operating characteristics.  

In comparison to other kinds of historical borrowing methods, prognostic covariate adjustment theoretically guarantees strict type-I error rate control and confidence interval coverage in general settings. In anything but the smallest of trials, there is no need for elaborate simulations to demonstrate the trial operating characteristics (as are usually required for methods that cannot theoretically guarantee control of type I error). Moreover, we provide a simple formula in theorem \ref{thm:pwr} that can be used to calculate power prospectively while accounting for the beneficial effect of prognostic covariate adjustment.

\subsection*{Data Availability}

{\small 
The data used in this study are available from the following sources, subject to their discretion.

Certain data used in the preparation of this article were obtained from the Alzheimer’s Disease Neuroimaging Initiative (ADNI) database (\href{url}{adni.loni.usc.edu}). The ADNI was launched in 2003 as a public-private partnership, led by Principal Investigator Michael W. Weiner, MD. The primary goal of ADNI has been to test whether serial magnetic resonance imaging (MRI), positron emission tomography (PET), other biological markers, and clinical and neuropsychological assessment can be combined to measure the progression of mild cognitive impairment (MCI) and early Alzheimer’s disease (AD). For up-to-date information, see \href{url}{www.adni-info.org}.

Certain data used in the preparation of this article were obtained from the Critical Path for Alzheimer's Disease (CPAD) database. In 2008, Critical Path Institute, in collaboration with the Engelberg Center for Health Care Reform at the Brookings Institution, formed the Coalition Against Major Diseases (CAMD), which was then renamed to CPAD in 2018. The Coalition brings together patient groups, biopharmaceutical companies, and scientists from academia, the U.S. Food and Drug Administration (FDA), the European Medicines Agency (EMA), the National Institute of Neurological Disorders and Stroke (NINDS), and the National Institute on Aging (NIA). CPAD currently includes over 200 scientists, drug development and regulatory agency professionals, from member and non-member organizations. The data available in the CPAD database has been volunteered by CPAD member companies and non-member organizations.

Certain data used in the preparation of this article were obtained from the University of California, San Diego Alzheimer’s Disease Cooperative Study Legacy database.
}

\subsection*{Acknowledgments}

{\small
We are grateful to Xinkun Nie and Oleg Sofrygin for enlightening conversations and to Rachael C. Aikens for feedback on a draft of this article.

Data collection and sharing for this project was funded in part by the Alzheimer's Disease Neuroimaging Initiative (ADNI) (National Institutes of Health Grant U01 AG024904) and DOD ADNI (Department of Defense award number W81XWH-12-2-0012). ADNI is funded by the National Institute on Aging, the National Institute of Biomedical Imaging and Bioengineering, and through generous contributions from the following: AbbVie, Alzheimer’s Association; Alzheimer’s Drug Discovery Foundation; Araclon Biotech; BioClinica, Inc.; Biogen; Bristol-Myers Squibb Company; CereSpir, Inc.; Cogstate; Eisai Inc.; Elan Pharmaceuticals, Inc.; Eli Lilly and Company; EuroImmun; F. Hoffmann-La Roche Ltd and its affiliated company Genentech, Inc.; Fujirebio; GE Healthcare; IXICO Ltd.; Janssen Alzheimer Immunotherapy Research \& Development, LLC.; Johnson \& Johnson Pharmaceutical Research \& Development LLC.; Lumosity; Lundbeck; Merck \& Co., Inc.; Meso Scale Diagnostics, LLC.; NeuroRx Research; Neurotrack Technologies; Novartis Pharmaceuticals Corporation; Pfizer Inc.; Piramal Imaging; Servier; Takeda Pharmaceutical Company; and Transition Therapeutics. The Canadian Institutes of Health Research is providing funds to support ADNI clinical sites in Canada. Private sector contributions are facilitated by the Foundation for the National Institutes of Health (\href{url}{www.fnih.org}). The grantee organization is the Northern California Institute for Research and Education, and the study is coordinated by the Alzheimer’s Therapeutic Research Institute at the University of Southern California. ADNI data are disseminated by the Laboratory for Neuro Imaging at the University of Southern California.

Data collection and sharing for this project was funded in part by the University of California, San Diego Alzheimer’s Disease Cooperative Study (ADCS) (National Institute on Aging Grant Number U19AG010483).
}

\bibliography{references}

\appendix

\section{Mathematical Results}

{

Throughout we assume enough regularity conditions for the asymptotic normality of M-estimators to hold. The details are found in chapter 5 (thm 5.23) of \citet{Van_der_Vaart2000-yx}.
}

\begin{lemma}[Rosenblum]
\label{thm:rosenblum}

The influence function for the linear regression treatment effect estimator we describe in section \ref{sec:approach} is $\psi = \psi_1 - \psi_0$ where

\begin{equation}
    \psi_w =
        \frac{W_w}{\pi_w} (Y-\hat\mu^*_w(X)) + 
        (\hat\mu^*_w(X) - \hat\mu^*_w)
\label{eq:influence}
\end{equation}

and $\hat\mu_w^*(X) = Z_w^\top \beta^*$ and $\hat\mu^*_w = \E{\hat\mu_w^*(X)}$. The parameters $\hat\beta^*$ are those that maximize the (model-based) likelihood in expectation (under the true law of the data). In other words, $\hat\mu_w^*(X)$ characterizes the linear model that comes as close as possible to the true conditional mean function $\mu_w(X) = \E{Y_w|X}$ and $\hat\mu^*_w$ is its mean value (averaged over $X$).
\end{lemma}

This follows from results in \citet{Robins:1994gq}. An accessible presentation for the case of generalized linear models is given in \citet{rosenblum-glm}.

\begin{definition}[Difference-in-means]
\label{def:diff}
The ``difference-in-means'' (or ``unadjusted'') estimator of $\tau = \mu_1-\mu_0$ is $\hat\tau_\Delta = \Ehat{Y|W_1} - \Ehat{Y|W_0}$.
\end{definition}

Note that throughout the appendix we omit the subscript $n$ on estimators. E.g. $\tau_\Delta$ is shorthand for $\tau_{\Delta, n}$ and our asymptotic statements refer to the sequence of estimators as $n$ becomes large.

\begin{lemma}
The difference-in-means estimator has asymptotic variance given by

\begin{align}
    n\V{\hat\tau_\Delta}
    &\pto \label{eq:asymtptotic-var-delta}
            \frac{\sigma_0^2}{\pi_0} + \frac{\sigma_1^2}{\pi_1}
\end{align}

where $\sigma_w = \V{Y_w}$.
\end{lemma}

\begin{proof}
This fact is well-known. One proof follows the outline of \ref{thm:ancovai-var} below taking $Z^\top = [1, W]$.
\end{proof}

\begin{definition}[ANCOVA I]
The ``ANCOVA I'' estimator of $\tau = \mu_1-\mu_0$ (denoted $\hat\tau_\numeral{1}$) is the effect estimated using a linear regression with predictors $Z^\top = [1, W, X^\top]$ and outcome $Y$.
\end{definition}

\begin{definition}[ANCOVA II]
The ``ANCOVA II'' estimator of $\tau = \mu_1-\mu_0$ (denoted $\hat\tau_\numeral{2}$) is the effect estimated using a linear regression with predictors $Z^\top = [1, \tilde W, \tilde X^\top, \tilde W \tilde X^T ]$ and outcome $\tilde Y$.
\end{definition}

The following two theorems (\ref{thm:ancovai-var} and \ref{thm:ancovaii-var}) are mild generalizations of or follow closely from results stated in \citet{leon} and \citet{yang-tsiatis}. Details are provided here for the reader's convenience.

\begin{theorem}
\label{thm:ancovai-var}

The ANCOVA I estimator is asymptotically unbiased for $\tau = \mu_1-\mu_0$ and has asymptotic variance given by

\begin{align}
    n\V{\hat\tau_\numeral{1}}
    &\pto \label{eq:asymtptotic-var-ancovaI}
            \frac{\sigma_0^2}{\pi_0} + \frac{\sigma_1^2}{\pi_1}
          + \left(\frac{1}{\pi_0 \pi_1}\right) \xi^\top V \xi 
          -2 \left(\frac{1}{\pi_0 \pi_1}\right) \xi_*^\top V \xi
\end{align}

where $\xi = \pi_0 \C{Y_0,X} + \pi_1 \C{Y_1,X}$, $\xi_* = \pi_0 \C{Y_1,X} + \pi_1 \C{Y_0,X}$, and $V = \V{X}^{-1}$

\end{theorem}

\begin{proof}

We begin by applying lemma \ref{thm:rosenblum}. Minimization of the expected log-likelihood shows that $\hat\beta^* = \E{ZZ^\top}^{-1}\E{ZY}$.  Some algebra
\footnote{
The identity $\E{AB} = \C{A,B} + \E{A}\E{B}$ and the fact that $W \perp Y_w, X$ and $W_wY = W_wY_w$ by our structural assumption (eq. \ref{rct-structure}) may be used to show that

\begin{equation*}
\E{ZZ^\top}^{-1}
=
\left[
\begin{array}{ccc}
    \frac{1}{\pi_0}+\eta^\top V\eta & 
    -\frac{1}{\pi_0} &
    -V\eta
    \\
    -\frac{1}{\pi_0} &
    \frac{1}{\pi_0\pi_1} &
    0
    \\
    -V\eta & 
    0 & 
    V \\
\end{array}
\right]
\quad \quad
\E{Z Y}
=
\left[
\begin{array}{c}
     \mu \\
     \pi_1\mu_1 \\
     \mu\eta + \xi \\
\end{array}
\right]
\end{equation*}

where $\eta = \E{X}$, $\mu = \E{Y}$, $\xi = \C{X,Y}$, and $V = \V{X}^{-1}$. The inverse is easiest to verify by computing and multiplying by $\E{ZZ^T}$.
}
demonstrates

\begin{equation}
\hat\beta^* = [\mu_0, \tau, (V\xi)^\top]^\top
\end{equation}

where $V = \V{X}^{-1}$, $\xi = \C{X,Y}$, and $\tau = \mu_1 - \mu_0$. Thus $\hat\mu_w^*(X) = \mu_0 + w\tau + \tilde X^\top V \xi = \mu_w + \tilde X^\top V \xi$. In this equation and from here on, let $\tilde X = X - \E{X}$. So clearly $\hat\mu_w^* = \mu_w$.  Then, from eq. \ref{eq:influence}, 

\begin{equation}
    \psi_w =
        \frac{W_w}{\pi_w} (Y-\mu_w)
      - \frac{\tilde W_w}{\pi_w} \underbrace{(\tilde X^\top V \xi)}_{-h_w(X)}
\end{equation}

Where $\tilde W_w = W_w - \pi_w$. An application of  \ref{thm:rosenblum} and some algebra gives

\begin{align}
    \psi_{\numeral{1}}
    & = 
     \underbrace{
        \underbrace{\frac{W_1}{\pi_1} (Y-\mu_1)}_{\psi_{1,\Delta}}
      - \underbrace{\frac{W_0}{\pi_0} (Y-\mu_0)}_{\psi_{0,\Delta}}
      }_{\psi_\Delta}
      - \underbrace{(W_1-\pi_1) 
            \underbrace{\frac{(\tilde X^\top V \xi)}{\pi_0 \pi_1}
            }_{-h(X)}
        }_{\phi}
\label{eq:inflence-lm}
\end{align}

It is known that all regular and asymptotically linear estimators of the treatment effect have an influence function of this form with $h(X)$ dependent on the choice of estimator \cite{leon, Tsiatis:2007vl}. 

By the theory of influence functions, our estimator has a limiting distribution \cite{Tsiatis:2007vl}

\begin{equation}
\sqrt{n}(\hat\tau_{\numeral{1}} - \tau) \xrightarrow{d} \mathcal N(0,\E{\psi_{\numeral{1}}^2})
\label{eq:if-clt}
\end{equation}

The asymptotic variance of $\hat\tau_{\numeral{1}}$ is thus $\E{\psi_{\numeral{1}}^2} = \E{(\psi_\Delta - \phi)^2} = \E{\psi_\Delta^2} - 2\E{\psi_\Delta \phi} + \E{\phi^2}$. The first term is the variance of the influence function for the difference-in-means (also called ``unadjusted'') estimator. It may be verified that this evaluates to $\E{\psi_\Delta^2} = \frac{\sigma_0^2}{\pi_0} + \frac{\sigma_1^2}{\pi_1}$ where $\sigma_w^2 = \V{Y_w}$. The variance of $\phi$ is

\begin{align}
    \E{\phi^2} 
    & = \E{\left(\frac{W_1-\pi_1}{\pi_0 \pi_1} \tilde X^\top V \xi\right)^2}
    \\
    & = 
        \frac{\E{(W_1-\pi_1)^2}}{\pi_0^2 \pi_1^2} 
        \xi^\top V \E{\tilde X \tilde X^\top} V \xi
    \\
    & = 
        \left(\frac{1}{\pi_0 \pi_1}\right) 
        \xi^\top V \xi
\end{align}

The covariance of the two terms involves the expectations $\E{(Y_w-\mu_w)\tilde X} = \C{Y_w, X} = \xi_w$ (note that $\xi = \pi_0\xi_0 + \pi_1 \xi_1$):

\begin{align}
    \E{\psi_\Delta \phi} 
    & = \E{\psi_{1,\Delta} \phi} - \E{\psi_{0,\Delta} \phi} \\
    & = 
        \frac{1}{\pi_1} \xi_1^\top V \xi
      - \frac{-1}{\pi_0} \xi_0^\top V \xi 
    \\
    & = 
        \left(\frac{1}{\pi_0 \pi_1}\right) 
        \xi_*^\top V \xi 
\end{align}

where we have introduced $\xi_* = \pi_1\xi_0 + \pi_0 \xi_1$. Assembling obtains the desired result.

\end{proof}

\begin{corollary}
When $X \in R$ (a single covariate), a consistent estimate of the sampling variance $\V{\hat\tau_\numeral{1}}$ is

\begin{align}
    \hat\nu_\numeral{1}^2
    &= 
        \frac{\hat\sigma_0^2}{n_0} + \frac{\hat\sigma_1^2}{n_1}
      + \frac{n_0 n_1}{n}
        \left(
            \frac{\hat\rho_0\hat\sigma_0}{n_1}
          + \frac{\hat\rho_1\hat\sigma_1}{n_0}
        \right)^2
      -2 \frac{n_0 n_1}{n}
        \left(
            \frac{\hat\rho_0\hat\sigma_0}{n_1}
          + \frac{\hat\rho_1\hat\sigma_1}{n_0}
        \right)
        \left(
            \frac{\hat\rho_0\hat\sigma_0}{n_0}
          + \frac{\hat\rho_1\hat\sigma_1}{n_1}
        \right)
\end{align}

where $\rho_w = \C{Y_w,X}/\sqrt{\V{X} \V{Y_w}}$ and the ``hat'' quantities are any consistent estimates of their respective population parameters.

\end{corollary}

\begin{proof}
This follows from the definitions and Slutsky's theorem.
\end{proof}

\begin{corollary}
If either $\pi_0 = \pi_1$ or $\xi_0 = \xi_1$, then 

\begin{align}
    n\V{\hat\tau_\numeral{1}}
    &\pto \label{eq:asymtptotic-var-ancovaI-equal-pi}
            \frac{\sigma_0^2}{\pi_0} + \frac{\sigma_1^2}{\pi_1}
          -\left(\frac{1}{\pi_0 \pi_1}\right) \xi_*^\top V \xi_* 
\end{align}

\end{corollary}

\begin{theorem}
\label{thm:ancovaii-var}

The ANCOVA II estimator is asymptotically unbiased for $\tau = \mu_1-\mu_0$ and has asymptotic variance given by

\begin{align}
    n\V{\hat\tau_\numeral{2}}
    &\pto \label{eq:asymtptotic-var-ancovaII}
            \frac{\sigma_0^2}{\pi_0} + \frac{\sigma_1^2}{\pi_1}
          - \left(\frac{1}{\pi_0 \pi_1}\right) \xi_*^\top V \xi_* 
\end{align}
\end{theorem}

\begin{proof}
Arguments similar to those in thm. \ref{thm:ancovai-var} show that the influence function for the GLM marginal effect estimator with this specification is identical to eq. \ref{eq:inflence-lm} except that $\xi = \pi_0\xi_0 + \pi_1\xi_1$ is replaced by $\xi_* = \pi_1\xi_0 + \pi_0 \xi_1$. Specifically $\psi_\numeral{2} = \psi_{1, \numeral{2}} - \psi_{0, \numeral{2}}$ with

\begin{equation}
    \psi_{w, \numeral{2}} =
        \frac{W_w}{\pi_w} (Y-\mu_w)
      - \frac{\tilde W_w}{\pi_w} \underbrace{(\tilde X^\top V \xi_*)}_{-h_w(X)}
\label{eq:ancovaii-if}
\end{equation}

The result follows from proceeding along the outline of thm. \ref{thm:ancovai-var}.
\end{proof}

\begin{corollary}
When $X \in R$ (a single covariate), a consistent estimate of the sampling variance $\V{\hat\tau_\numeral{2}}$ is

\begin{align}
    \hat\nu_\numeral{2}^2
    &= 
        \frac{\hat\sigma_0^2}{n_0} + \frac{\hat\sigma_1^2}{n_1}
      - \frac{n_0 n_1}{n}
        \left(
            \frac{\hat\rho_0\hat\sigma_0}{n_0}
          + \frac{\hat\rho_1\hat\sigma_1}{n_1}
        \right)^2
\end{align}
\end{corollary}

\begin{corollary}
\label{thm:add-covar}
Adding covariates to the ANCOVA II estimator can only decrease its asymptotic variance.
\end{corollary}

\begin{proof}
Consider using covariates $X$ with variance $\Sigma_x$ and covariance with $Y_w$ of $\xi_{w,x}$ vs. a set of covariates $[X, M]$ ($M \in \mathbb R$) {such that $M$ is not a linear combination of the variables in $X$. Let} $\C{X,M} = \zeta$, $\V{M} = \sigma_m^2$ and $\C{Y_w, M} = \xi_{w,m}$. Let $\xi_{m*} = \pi_0 \xi_{1,m} + \pi_1 \xi_{0,m}$ and $\xi_{x*} = \pi_0 \xi_{1,x} + \pi_1 \xi_{0,x}$. From eq. \ref{eq:asymtptotic-var-ancovaII} and some matrix algebra the difference in asymptotic variance between these two estimators is

\begin{align}
    -\left(
    \frac{1}{\pi_0 \pi_1} 
    \right)
    \frac{
        (\xi_m - \xi_x^\top \Sigma_x^{-1} \zeta)^2
    }{
        \sigma_m^2 - \zeta^\top \Sigma_x^{-1} \zeta
    }
\end{align}

The denominator must be positive because $\V{X,M} \ge 0$, $\V{X} \ge 0$ implies $\operatorname{det}(\V{X,M}) = \operatorname{det}(\Sigma_x^{-1}) (\sigma_m^2 - \zeta^\top \Sigma_x^{-1} \zeta) \ge 0$. 
\end{proof}

\begin{theorem}
\label{thm:comparison}
ANCOVA II is a more efficient estimator than ANCOVA I or difference-in-means. ANCOVA I may or may not be more efficient than difference-in-means (unless $\pi_0=\pi_1=0.5$ or $\xi_0 = \xi_1$, in which case it is as efficient as ANCOVA II). In a slight abuse of notation,

\begin{align}
\V{\hat\tau_\numeral{2}} & \le \V{\hat\tau_\numeral{1}} \\
\V{\hat\tau_\numeral{2}} & \le \V{\hat\tau_\Delta} \\
\V{\hat\tau_\numeral{1}} & \nleq \V{\hat\tau_\Delta} \\
\pi_0=\pi_1 \implies \V{\hat\tau_\numeral{1}} & = \V{\hat\tau_\numeral{2}} 
\end{align}
\end{theorem}

\begin{proof}
$\V{\hat\tau_\numeral{2}} \le \V{\hat\tau_\numeral{1}}$ because eq. \ref{eq:asymtptotic-var-ancovaII} subtracted from eq. \ref{eq:asymtptotic-var-ancovaI} is $(V^{1/2}(\xi - \xi_*))^2/(\pi_0 \pi_1) \ge 0$. $\V{\hat\tau_\numeral{2}} \le \V{\hat\tau_\Delta}$ is self-evident from eq. \ref{eq:asymtptotic-var-ancovaII}. To show $\V{\hat\tau_\numeral{1}} \nleq \V{\hat\tau_\Delta}$ we rely on an example: using $X \in R$ with $\pi_1 = 5/6$ (so $\pi_0 =1/6$), $\xi_1=4$ and $\xi_0=1$ in eq. \ref{eq:asymtptotic-var-ancovaI} gives a positive addition to $\V{\hat\tau_\Delta}$.
\end{proof}

\begin{lemma}
\label{thm:opt}
Consider using the ANCOVA II estimator with an arbitrary (multivariate) transformation of the covariates $f(X)$ in place of the raw covariates $X$. Among all fixed transformations $f(X)$, the transformation $[\mu_0(X), \mu_1(X)]^\top$ is optimal in terms of efficiency. Furthermore, the estimator is semiparametric efficient: the ANCOVA II estimator with $[\mu_0(X), \mu_1(X)]^\top$ used as the vector of covariates has the lowest possible asymptotic variance among all regular and asymptotically linear estimators with access to the covariates $X$. 
\end{lemma}

Consider replacing $X$ in the interacted linear model (ANCOVA II) with an arbitrary fixed (possibly multivariate) function of the covariates $f(X)$. By eq. \ref{eq:ancovaii-if} and our definitions of $\xi_*$ and $V$ the influence function for this estimator is $\psi = \psi_1 - \psi_0$ with

\begin{equation}
    \psi_w =
        \frac{W_w}{\pi_w} (Y-\mu_w)
      - \frac{\tilde W_w}{\pi_w} \underbrace{
        \left( (f(X) - \E{f(X)})^\top V_f \xi_{f*} \right)
      }_{-h_{w}(X)}
\label{eq:ancovaii-f-if}
\end{equation}

where $\xi_{f*} = \pi_1 \C{Y_0, f(X)} + \pi_0 \C{Y_1, f(X)}$ and $V_f = \V{f(X)}^{-1}$. Consider now using the special transformation $f(X)^\top = [\mu_0(X), \mu_1(X)]$ where $\mu_w(X) = \E{Y_w|X}$. Note that $\C{Y_w, \mu_w(X)} = \V{\mu_w(X)}$ and $\C{Y_1, \mu_0(X)} = \C{\mu_1(X), \mu_0(X)}$ by an orthogonal decomposition of $Y_w$.\footnote{
Let $R = Y_w - \E{Y_w|X}$ be the part of $Y_w$ orthogonal to $\E{Y_w|X} = \mu_w(X)$ so that $Y_w = R + \mu_w(X)$. Note that

\begin{equation*}
\C{Y_w, f(X)} = \C{R, f(X)} + \C{\mu_w(X), f(X)}
\label{eq:decomp}
\end{equation*}

Now we prove a known result that $\C{R, f(X)} = 0$ for any function $f$:

\begin{equation*}
    \begin{array}{rcl}
         \C{R, f(X)} &=& \E{(R-\underset{0}{\underbrace{\E{R}}})(
         \underset{\tilde f(X)}{\underbrace{f(X) - \E{f(X)}})}} \\
         &=& \E{(Y_w-\E{Y_w|X})\tilde f(X)}\\
         &=& \E{\E{(Y_w-\E{Y_w|X})\tilde f(X)|X}}\\
         &=& \E{(\E{Y_w|X}-\E{Y_w|X})\tilde f(X)}\\
         &=& 0
    \end{array}
\end{equation*}
}
Plugging these in and performing the appropriate algebra shows that $V_f \xi_{f*}$ in this case is $[\pi_1, \pi_0]^\top$ so $h_w(X)$ in \ref{eq:ancovaii-f-if} is $\pi_0 (\mu_1(X) - \mu_1) + \pi_1 (\mu_0(X) - \mu_0)$. A little algebra shows

\begin{align}
    \psi 
    &= \psi_1 - \psi_0 \\
    &= 
        \frac{W_1}{\pi_1} (Y-\mu_1)
      - \frac{W_0}{\pi_0} (Y-\mu_0)
      - (W_1 - \pi_1)\left[\frac
      {\pi_0 (\mu_1(X) - \mu_1) + \pi_1 (\mu_0(X) - \mu_0)}
      {\pi_0\pi_1}
      \right]
\end{align}

The result is precisely the \textit{efficient influence function} for the treatment effect \cite{leon, Tsiatis:2007vl}. It is known that no regular and asymptotically linear (RAL) estimator (which essentially all practical and reasonable estimators are) can be more efficient than any estimator with this influence function. 

\begin{corollary}
\label{thm:constant-effect}
Presume a constant treatment effect: $\mu_1(X) = \mu_0(X) + \tau$. Then the ANCOVA II analysis that uses $\mu_0(X)$ in the role of $X$ has the lowest possible asymptotic variance among all regular and asymptotically linear estimators with access to the covariates $X$. 
\end{corollary}

\begin{proof}
$\mu_1(X) = \mu_0(X) + \tau$ implies $\C{\mu_0(X), \mu_1{X}} = \V{\mu_0(X)} = \V{\mu_1(X)}$. Following the outline for the proof of thm. \ref{thm:opt} above shows that the influence function for the ANCOVA II estimator with $\mu_0(X)$ as the single covariate is 

\begin{align}
\label{eq:ancovaii-if-const}
    \psi =
        \frac{W_1}{\pi_1} (Y-\mu_1)
      - \frac{W_0}{\pi_0} (Y-\mu_0)
      - (W_1 - \pi_1)\left[\frac
      {\mu_0(X) - \mu_0}
      {\pi_0\pi_1}
      \right]
\end{align}

which is the same as the efficient influence function when $\mu_1(X) = \mu_0(X) + \tau$.
\end{proof}

\begin{corollary}
\label{thm:constant-effect-ancovaI}
Corollary \ref{thm:constant-effect} also holds when the ANCOVA II estimator is replaced by the ANCOVA I estimator.
\end{corollary}

\begin{proof}
Thm. \ref{thm:comparison} establishes that ANCOVA I is as efficient as ANCOVA II when $\C{m(X), Y_0} = \xi_0 = \xi_1 = \C{m(X), Y_1}$. A constant treatment effect means that $\mu_1(X) = \mu_0(X)+\tau$ and this ensures the equality of the covariances.
\end{proof}

The following lemma is required for the proof that proceeds it.

{

\begin{lemma}
\label{thm:expectation-convergence}
Let $f: \mathcal X \to \mathbb R$ be a bounded function on a compact set $\mathcal X$ and let $\hat f_n: \mathcal X \to \mathbb R$ be a sequence of uniformly bounded random functions such that $|f(X) - \hat f_n(X)| \overset{L_2}{\to} 0$. Let $X \in \mathcal X$ be a random variable independent of $\hat f_n$. Then $\E[X]{\hat f_n(X)} \pto \E{f(X)}$, 
$\C[X]{f(X), \hat f_n(X)} \pto \V{f(X)}$,
and 
$\V[X]{\hat f_n(X)} \pto \V{f(X)}$.

\begin{proof}

$\hat f_n$ and $X$ are independent, so let their joint distribution factor into $P'_n$ and $P$. Now

\begin{align*}
    &
    \int\left(
        \E[X]{\hat f_n(X)} - 
        \E{f(X)}
    \right)^2 dP_n' 
    \\
    &= 
    \int\left[ 
        \int \hat f_n(X) dP - 
        \int f(X) dP 
    \right]^2 dP_n' 
    \\
    &= 
    \int\left[ 
        \int \hat f_n(X) - 
        f(X) dP 
    \right]^2 dP_n' 
    \\
    &\le 
    \int
        \int \left[ \hat f_n(X) - 
        f(X)\right]^2 dP 
    dP_n' \quad \text{(Jensen's inequality)}
    \\    
    &\to 0 
\end{align*}

The final convergence holds by our assumption that $|f(X) - \hat f_n(X)| \overset{L_2}{\to} 0$. This shows $\E[X]{\hat f_n(X)} \overset{L_2}{\to} \E{f(X)}$ and convergence in probability follows.

Taking advantage of the fact that $|f|, |f_n| \le b$ are bounded we can make similar arguments to show that 
$\E[X]{f(X)\hat f_n(X)} \pto \E[X]{f(X)^2}$ and $\E[X]{\hat f_n(X)^2} \pto \E[X]{f(X)^2}$. Slutsky's theorem and the definition of covariance and variance then imply $\C[X]{f(X), \hat f_n(X)} \pto \C{f(X), f(X)}$
and 
$\V[X]{\hat f_n(X)} \pto \V{f(X)}$ as desired.

\end{proof}
\end{lemma}

% \begin{corollary}
% \label{thm:expectation-convergence-corollary}
% Under the conditions of the above lemma, 
% $
% \underset{x}{\text{sup}}
% \ 
% \left|
% f(x) 
% - \hat f_n(x)
% \frac{\C[X]{f(X), \hat f_n(X)}}{\V[X]{\hat f_n(X)}}
% \right| 
% \pto 0
% $.

% \begin{proof}
% Let $B_n = \frac{\C[X]{f(X), \hat f_n(X)}}{\V[X]{\hat f_n(X)}}$. By the above lemma and Slutsky's theorem, $B_n \pto 1$. 
% Note $
% \left|f(x) - \hat f_n(x)B_n\right| 
% \le
% \left|f(x) - \hat f_n(x)\right| 
% +
% b \left|1 - B_n\right| 
% $ 
% by the triangle inequality and the fact that $|\hat f_n(x)| < b$. Therefore 

% \begin{equation*}
% \begin{split}
% P\left(
% \underset{x}{\text{sup}}
% \ 
% \left|
% f(x) 
% - \hat f_n(x)
% B_n
% \right| 
% > \epsilon
% \right)
% & \le
% P\left(
% \underset{x}{\text{sup}}
% \ 
% \left\{
% \left|f(x) - \hat f_n(x)\right| 
% +
% b \left|1 - B_n\right| 
% \right\}
% > \epsilon
% \right)
% \\
% &=
% P\left(
% \underset{x}{\text{sup}}
% \ 
% \left\{
% \left|f(x) - \hat f_n(x)\right| 
% \right\}
% +
% b \left|1 - B_n\right| 
% > \epsilon
% \right)
% \\
% \end{split}
% \end{equation*}

% The right-hand side can be made arbitrarily small because 
% $
% \underset{x}{\text{sup}}
% \left|f(x) - \hat f_n(x)\right| 
% \pto 0
% $
% and 
% $B_n \pto 1$ 
% together with Slutsky's theorem imply that 
% $
% \underset{x}{\text{sup}}
% \ 
% \left\{
% \left|f(x) - \hat f_n(x)\right| 
% \right\}
% +
% b \left|1 - B_n\right| 
% \pto 0
% $.
% \end{proof}
% \end{corollary}

\begin{corollary}
\label{thm:expectation-convergence-corollary}
Let $\V[X]{\hat f_n(X)} > \epsilon > 0$. Under the conditions of the above lemma, 
$
\left|
f(x) 
- \hat f_n(x)
\frac{\C[X]{f(X), \hat f_n(X)}}{\V[X]{\hat f_n(X)}}
\right| 
\overset{L_2}{\to} 0
$.

\begin{proof}
Let $B_n = \frac{\C[X]{f(X), \hat f_n(X)}}{\V[X]{\hat f_n(X)}}$. By the above lemma, our assumption that $\V[X]{\hat f_n(X)} > \epsilon > 0$, and Slutsky's theorem, $B_n \pto 1$. Together with the uniform bound on $\V[X]{\hat f_n(X)}$ and Cauchy-Schwarz this is also enough to ensure that $(1-B_n) \overset{L_2}{\to} 0$

Now note $
\left|f(x) - \hat f_n(x)B_n\right| 
\le
\left|f(x) - \hat f_n(x)\right| 
+
b \left|1 - B_n\right| 
$ 
by the triangle inequality and the fact that $|\hat f_n(x)| < b$. Thus

\begin{align*}
    & 
    \E{(f(X) - \hat f_n(X) B_n)^2}
    \\
    &=
    \underbrace{\E{(f(X) - \hat f_n(X))^2}}_{\to 0 \ \ \text{(by assumption)}} +
    \underbrace{b^2\E{(1-B_n)^2}}_{\to 0 \ \ \text{(shown above)}} + 
    b\E{\underbrace{(f(X)-f_n(X))}_{\le 2b}(1-B_n)}
    \\
    &= o(1) + o(1) + 2b^2 \E{1-B_n}
    \\
    &\to 0
\end{align*}

as desired.

\end{proof}
\end{corollary}
}

\begin{theorem}
\label{thm:asymptotic}
Presume $X$ has compact support and there is a constant treatment effect: $\mu_1(X) = \mu_0(X) + \tau$ with $|\mu_0(x)| < b$ bounded. Let $m(x)$ be a (random) function learned from the external data $(\bm Y', \bm X')_{n'}$ such that $|m(x)| < b$ is also bounded and {$|m(X) - \mu_0(X)| \overset{L_2}{\to} 0$ so that the learned model approaches the truth in mean-squared error} as $n' \to \infty$. If the number of trial samples $n$ grows in tandem with the size of the historical data $n'$ (i.e. $n = O(n')$), then the ANCOVA II analysis that uses the learned model $m(X)$ in the role of $X$ has the lowest possible asymptotic variance among all regular and asymptotically linear estimators with access to the covariates $X$. 

\begin{proof}
Define our estimator of interest as the ANCOVA II estimator that uses the learned model $m(X)$ in place of the covariates $X$ { if $m(X)$ is not numerically constant up to some machine precision and otherwise as the difference-in-means estimator}. Denote this estimator $\hat\tau$ (omitting the II subscript for the duration of this proof). Define the ``oracle'' estimator as the equivalent estimator that uses the true conditional mean $\mu_0(X)$ instead of the estimate $m(X)$ and denote this estimator $\hat\tau^*$. The oracle estimator is obviously infeasible in practice because $\mu_0(\cdot)$ is not known. Corollary  \ref{thm:constant-effect} proves that the oracle estimator is semiparametric efficient (i.e. has the lowest possible asymptotic variance among regular and asymptotically linear estimators). Thus, letting $\nu_*^2$ denote the optimal asymptotic variance, we have that $\sqrt{n}(\hat\tau^*-\tau) \rightsquigarrow N(0, \nu_*^2)$. If we can show that $\sqrt{n}(\hat\tau - \hat\tau^*) \pto 0$, then Slutsky's theorem and the delta method imply that $\hat\tau$ has the same asymptotic properties as $\hat\tau^*$, i.e. $\sqrt{n}(\hat\tau - \tau) \rightsquigarrow N(0, \nu_*^2)$. In other words, since the oracle estimator is efficient with a known asymptotic variance, the feasible estimator is also efficient and has the same asymptotic variance because the two are asymptotically equivalent.

Showing $\sqrt{n}(\hat\tau - \hat\tau^*) \pto 0$ requires an intermediate estimator that is asymptotically equivalent to $\hat\tau$. Using the assumption of the constant effect and eq. \ref{eq:ancovaii-if} from theorem \ref{thm:ancovaii-var} we can show (with an application of the law of total variance) that the influence function for $\hat\tau$ using some fixed $m(\cdot)$ is $\psi = \psi_1-\psi_0$ with

\begin{equation}
    \psi_{w} =
        \frac{W_w}{\pi_w} (Y-\mu_w)
      - \frac{\tilde W_w}{\pi_w} 
      \left(
      \left(m(X) - \E[X]{m(X)}\right)
      \frac{\C[X]{m(X), \mu_0(X)}}{\V[X]{m(X)}}
      \right)
\end{equation}

where $\E[X]{m(X)}$ denotes that the expectation (or variance or covariance) is taken only with respect to $X$, i.e. $m(\cdot)$ is considered fixed.

Let $\check\tau = \Ehat{\psi+\tau}$ and let $\check\tau^* = \Ehat{\psi^*+\tau}$ where $\psi^*$ is the influence function above with $\mu_0(\cdot)$ substituted for $m(\cdot)$. Note that $\hat\tau$ and $\check\tau$ share the same influence function so we must have that $\sqrt{n}(\hat\tau - \check\tau) \pto 0$. Similarly, $\sqrt{n}(\hat\tau^* - \check\tau^*) \pto 0$. Therefore if $\sqrt{n}(\check\tau - \check\tau^*) \pto 0$, then we have $\sqrt{n}(\hat\tau - \hat\tau^*) \pto 0$ as desired. This is useful because the estimator $\check\tau$ and its oracle counterpart $\check\tau^*$ are easier to work with.

To wit, consider the difference $\check\tau - \check\tau^* = \Ehat{(\psi_1 - \psi_0) - (\psi_1^* - \psi_0^*)}$. So all we need to show the desired convergence $\sqrt{n}(\check\tau - \check\tau^*) \pto 0$ is to show $\sqrt{n} \Ehat{\psi_w-\psi_w^*} \pto 0$. Expanding,

\begin{equation}
\begin{split}
\Ehat{\psi_w-\psi_w^*} &= 
\frac{1}{n} \sum^n_i
\frac{\tilde W_{w,i}}{\pi_w} 
      \left(
      \left(\mu_0(X_i) - \E[X]{\mu_0(X)}\right)
      \frac{\C[X]{\mu_0(X), \mu_0(X)}}{\V[X]{\mu_0(X)}}
      -
      \left(m(X_i) - \E[X]{m(X)}\right)
      \frac{\C[X]{m(X), \mu_0(X)}}{\V[X]{m(X)}}
      \right) \\
% &= 
% \frac{1}{n} \sum^n_i
% \frac{\tilde W_w}{\pi_w} 
%       \left(
%       \left(\mu_0(X) - \E{\mu_0(X)}\right)
%       -
%       \left(m(X) - \E{m(X)}\right)
%       \frac{\C{m(X), \mu_0(X)}}{\V{m(X)}}
%       \right) \\
&= 
\frac{1}{n} \sum^n_i
\frac{\tilde W_{w,i}}{\pi_w} 
      \left(
      \mu_0(X_i)
      -
      m(X_i)
      B
      \right)
- \frac{1}{n} \sum^n_i
\frac{\tilde W_{w,i}}{\pi_w} 
      \left(
      \mu_0
      -
      m B
      \right)
      \\
\end{split}
\label{eq:oracle-diff}
\end{equation}

where we've abbreviated $B = \frac{\C[X]{m(X), \mu_0(X)}}{\V[X]{m(X)}}$ and $m = \E[X]{m(X)}$. Our plan is to show
that both of these terms $L^2$-converge to 0 at the $\sqrt{n}$ rate so that they both converge in probability in that rate, as does their sum (which is what we want). To show $L^2$ convergence for the first term, we must consider the expression 

\begin{equation}
\E{
\left(
\sqrt{n} \frac{1}{n} \sum^n_i
\frac{\tilde W_w}{\pi_w} 
      (
      \mu_0(X)
      -
      m(X)
      B
      )
\right)^2
}
\end{equation}

And show it converges to 0. Recalling that $m$ itself is random (depends on the external data $(\bm X' \bm Y')$), but independent of the trial data $(\bm X, \bm W, \bm Y)$, note that we can treat $m(\cdot)$ as if it were a fixed function and $B$ as a fixed constant if we condition on the external data. After conditioning, the quantity inside the parentheses is IID and has mean zero because its $\mu_0(X)-m(X)B$ and $\tilde W_w$ (by randomization) and because $\E{\tilde W_w}=0$. Therefore the quantity above is

\begin{equation}
\begin{split}
\E{
n
\E{
\left(
\frac{1}{n} \sum^n_i
\frac{\tilde W_w}{\pi_w} 
      (
      \mu_0(X)
      -
      m(X)
      B
      )
\right)^2
\bigg|\bm X', \bm Y'}
}
&=
\E{
n
\V{
\left(
\frac{1}{n} \sum^n_i
\frac{\tilde W_w}{\pi_w} 
      (
      \mu_0(X)
      -
      m(X)
      B
      )
\right)
\bigg|\bm X', \bm Y'}
} 
\\
&=
\E{
\frac{n}{n}
\V{
\left(
\frac{\tilde W_w}{\pi_w} 
      (
      \mu_0(X)
      -
      m(X)
      B
      )
\right)
\bigg|\bm X', \bm Y'}
} 
\\
&=
\frac{1-\pi_w}{\pi_w} 
\E{
\left(
    \mu_0(X) - m(X) B
\right)^2
} 
\\
\end{split}
\end{equation}

where we've used the fact that the summands are IID to pass the variance through the sum and effectively gain the $1/n$ required to cancel the $n$. The same argument shows that the equivalent for the second term in eq. \ref{eq:oracle-diff} is 
$
\frac{1-\pi_w}{\pi_w} 
\E{
\left(
    \mu_0 - m B
\right)^2
} 
$ (note $m$ and $B$ are random here). 

{
To complete the proof we invoke corollary \ref{thm:expectation-convergence-corollary} in combination with our assumptions $|m(x)| < b$, $|\mu_0(x)| < b$ and $|m(X) - \mu_0(X)| \overset{L_2}{\to} 0$ to arrive at the fact that $|m(x)B - \mu_0(x)| \overset{L_2}{\to} 0$ and $|mB - \mu_0| \overset{L_2}{\to} 0$. The condition that $\V[X]{\hat f_n(X)}$ in \ref{thm:expectation-convergence-corollary} is automatically satisfied because we only include the prognostic score in the regression if it has nonzero variance
}. Thus the expectations 
$
\frac{1-\pi_w}{\pi_w} 
\E{
\left(
    \mu_0(X) - m(X) B
\right)^2
} 
$
and
$
\frac{1-\pi_w}{\pi_w} 
\E{
\left(
    \mu_0 - m B
\right)^2
} 
$
converge to 0 as desired. 

\end{proof}
\end{theorem}

\begin{corollary}
\label{thm:asymptotic-I}
Theorem \ref{thm:asymptotic} also holds for the ANCOVA I estimator.

\begin{proof}
In the case of a constant treatment effect ANCOVA I and ANCOVA II have the same asymptotic variance (thm. \ref{thm:constant-effect-ancovaI}). The result follows immediately.
\end{proof}
\end{corollary}

\section{Estimating $\sigma_w^2$ and $\rho_w$ for power calculations}
\label{appx:pwr}

One method for obtaining estimates for the marginal potential outcome variances ($\sigma^2_w$) and potential outcome-prognostic score correlations ($\rho_w$) is to use prior data, for example data from the placebo control arm of a previous trial performed on a similar population (separate from the data used to train the prognostic model). In this case we presume we have access to a vector $\bm Y'' = [Y''_1 \dots Y''_{n''}]$ of outcomes for these subjects and their corresponding prognostic scores $\bm M'' = [M''_1 \dots M''_{n''}]$, calculated by applying the prognostic model $m$ to each subject's vector of baseline covariates $X$, i.e. $M''_i = m(X''_i)$. 

The control-arm marginal outcome variance $\sigma^2_0$ can be estimated with the usual estimator 
$$
\hat\sigma^2_0 = \frac{1}{n'' - 1} \sum (Y''_i- \bar Y'')^2
$$

The correlation $\rho_0$ between $M''$ and $Y''$ can be estimated by 
$$
\hat\rho_0 = 
\frac{
    \sum\left(Y''_{i}-\bar{Y}''\right)\left(M''_{i}-\bar{M}''\right)
    }{
    \sqrt{\sum\left(Y''_{i}-\bar{Y}''\right)^{2}\sum\left(M''_{i}-\bar{M}''\right)^{2}}
    }
$$
which is the usual sample correlation coefficient. These values may be inflated ($\sigma^2_0$) or deflated  ($\rho_0$) in order to provide more conservative estimates of power.

The corresponding values for the treatment arm can rarely be estimated from data because treatment-arm data for the experimental treatment is likely to be scarce or unavailable. It is therefore prudent to assume $\sigma^2_0 = \sigma^2_1$ and $\rho_0 = \rho_1$, the latter which holds exactly if the effect of treatment is constant across the population. It may also be prudent (and conservative) to assume a slightly higher value for $\sigma^2_1$ and a slightly smaller value for $\rho_1$ relative to their control-arm counterparts in the absence of data to the contrary.

\section{Additional simulation results}

Here we detail a full set of simulation results using additional specifications for the regression estimators. ``Covariates'' indicates whether the raw covariates were adjusted for. ``Prognostic score'' indicates whether any prognostic score was used, and, if so, whether it was estimated from a training dataset or whether the true value was used. ``Interactions'' specifies whether treatment $\times$ (covariates and/or prognostic score) interactions were used. ``SE'' indicates the standard deviation of the mean squared error.

\begin{longtable}{llllrr}
\toprule
     scenario & covariates & prognostic score &  interaction &      mse &       se \\
\midrule
\endfirsthead

\toprule
     scenario & covariates & prognostic score &  interaction &      mse &       se \\
\midrule
\endhead
\midrule
\multicolumn{6}{r}{{Continued on next page}} \\
\midrule
\endfoot

\bottomrule
\endlastfoot
     baseline &      False &             None &         True & 7.64e-02 & 1.08e-03 \\
     baseline &      False &             None &        False & 7.64e-02 & 1.08e-03 \\
     baseline &      False &        Estimated &         True & 1.76e-02 & 2.46e-04 \\
     baseline &      False &        Estimated &        False & 1.75e-02 & 2.45e-04 \\
     baseline &      False &           Oracle &         True & 7.69e-03 & 1.09e-04 \\
     baseline &      False &           Oracle &        False & 7.69e-03 & 1.09e-04 \\
     baseline &       True &             None &         True & 5.07e-02 & 7.18e-04 \\
     baseline &       True &             None &        False & 5.04e-02 & 7.14e-04 \\
     baseline &       True &        Estimated &         True & 1.74e-02 & 2.46e-04 \\
     baseline &       True &        Estimated &        False & 1.73e-02 & 2.44e-04 \\
     baseline &       True &           Oracle &         True & 7.85e-03 & 1.11e-04 \\
     baseline &       True &           Oracle &        False & 7.85e-03 & 1.11e-04 \\
   surrrogate &      False &             None &         True & 7.47e-02 & 1.05e-03 \\
   surrrogate &      False &             None &        False & 7.47e-02 & 1.05e-03 \\
   surrrogate &      False &        Estimated &         True & 4.05e-02 & 5.69e-04 \\
   surrrogate &      False &        Estimated &        False & 4.03e-02 & 5.66e-04 \\
   surrrogate &      False &           Oracle &         True & 8.25e-03 & 1.18e-04 \\
   surrrogate &      False &           Oracle &        False & 8.24e-03 & 1.18e-04 \\
   surrrogate &       True &             None &         True & 5.03e-02 & 7.09e-04 \\
   surrrogate &       True &             None &        False & 5.00e-02 & 7.04e-04 \\
   surrrogate &       True &        Estimated &         True & 3.75e-02 & 5.27e-04 \\
   surrrogate &       True &        Estimated &        False & 3.72e-02 & 5.23e-04 \\
   surrrogate &       True &           Oracle &         True & 8.41e-03 & 1.20e-04 \\
   surrrogate &       True &           Oracle &        False & 8.41e-03 & 1.20e-04 \\
      shifted &      False &             None &         True & 7.65e-02 & 1.10e-03 \\
      shifted &      False &             None &        False & 7.65e-02 & 1.10e-03 \\
      shifted &      False &        Estimated &         True & 6.79e-02 & 9.62e-04 \\
      shifted &      False &        Estimated &        False & 6.79e-02 & 9.62e-04 \\
      shifted &      False &           Oracle &         True & 8.20e-03 & 1.15e-04 \\
      shifted &      False &           Oracle &        False & 8.20e-03 & 1.15e-04 \\
      shifted &       True &             None &         True & 5.03e-02 & 7.11e-04 \\
      shifted &       True &             None &        False & 5.00e-02 & 7.05e-04 \\
      shifted &       True &        Estimated &         True & 4.91e-02 & 6.97e-04 \\
      shifted &       True &        Estimated &        False & 4.86e-02 & 6.90e-04 \\
      shifted &       True &           Oracle &         True & 8.34e-03 & 1.17e-04 \\
      shifted &       True &           Oracle &        False & 8.34e-03 & 1.17e-04 \\
       strong &      False &             None &         True & 7.73e-02 & 1.08e-03 \\
       strong &      False &             None &        False & 7.73e-02 & 1.08e-03 \\
       strong &      False &        Estimated &         True & 1.85e-02 & 2.65e-04 \\
       strong &      False &        Estimated &        False & 1.85e-02 & 2.64e-04 \\
       strong &      False &           Oracle &         True & 8.16e-03 & 1.16e-04 \\
       strong &      False &           Oracle &        False & 8.16e-03 & 1.16e-04 \\
       strong &       True &             None &         True & 5.14e-02 & 7.18e-04 \\
       strong &       True &             None &        False & 5.11e-02 & 7.13e-04 \\
       strong &       True &        Estimated &         True & 1.84e-02 & 2.62e-04 \\
       strong &       True &        Estimated &        False & 1.82e-02 & 2.59e-04 \\
       strong &       True &           Oracle &         True & 8.33e-03 & 1.18e-04 \\
       strong &       True &           Oracle &        False & 8.32e-03 & 1.18e-04 \\
       linear &      False &             None &         True & 3.49e-02 & 4.83e-04 \\
       linear &      False &             None &        False & 3.49e-02 & 4.83e-04 \\
       linear &      False &        Estimated &         True & 9.64e-03 & 1.38e-04 \\
       linear &      False &        Estimated &        False & 9.64e-03 & 1.38e-04 \\
       linear &      False &           Oracle &         True & 8.20e-03 & 1.16e-04 \\
       linear &      False &           Oracle &        False & 8.20e-03 & 1.16e-04 \\
       linear &       True &             None &         True & 8.37e-03 & 1.18e-04 \\
       linear &       True &             None &        False & 8.37e-03 & 1.18e-04 \\
       linear &       True &        Estimated &         True & 8.39e-03 & 1.19e-04 \\
       linear &       True &        Estimated &        False & 8.39e-03 & 1.19e-04 \\
       linear &       True &           Oracle &         True & 8.37e-03 & 1.18e-04 \\
       linear &       True &           Oracle &        False & 8.37e-03 & 1.18e-04 \\
heterogeneous &      False &             None &         True & 5.54e-02 & 7.76e-04 \\
heterogeneous &      False &             None &        False & 5.54e-02 & 7.76e-04 \\
heterogeneous &      False &        Estimated &         True & 2.30e-02 & 3.23e-04 \\
heterogeneous &      False &        Estimated &        False & 2.32e-02 & 3.25e-04 \\
heterogeneous &      False &           Oracle &         True & 2.29e-02 & 3.20e-04 \\
heterogeneous &      False &           Oracle &        False & 2.32e-02 & 3.24e-04 \\
heterogeneous &       True &             None &         True & 2.99e-02 & 4.30e-04 \\
heterogeneous &       True &             None &        False & 2.98e-02 & 4.29e-04 \\
heterogeneous &       True &        Estimated &         True & 2.13e-02 & 3.01e-04 \\
heterogeneous &       True &        Estimated &        False & 2.19e-02 & 3.08e-04 \\
heterogeneous &       True &           Oracle &         True & 1.89e-02 & 2.69e-04 \\
heterogeneous &       True &           Oracle &        False & 1.98e-02 & 2.81e-04 \\
\end{longtable}

\begin{figure}[h]
\centering
\includegraphics[width=1\textwidth]{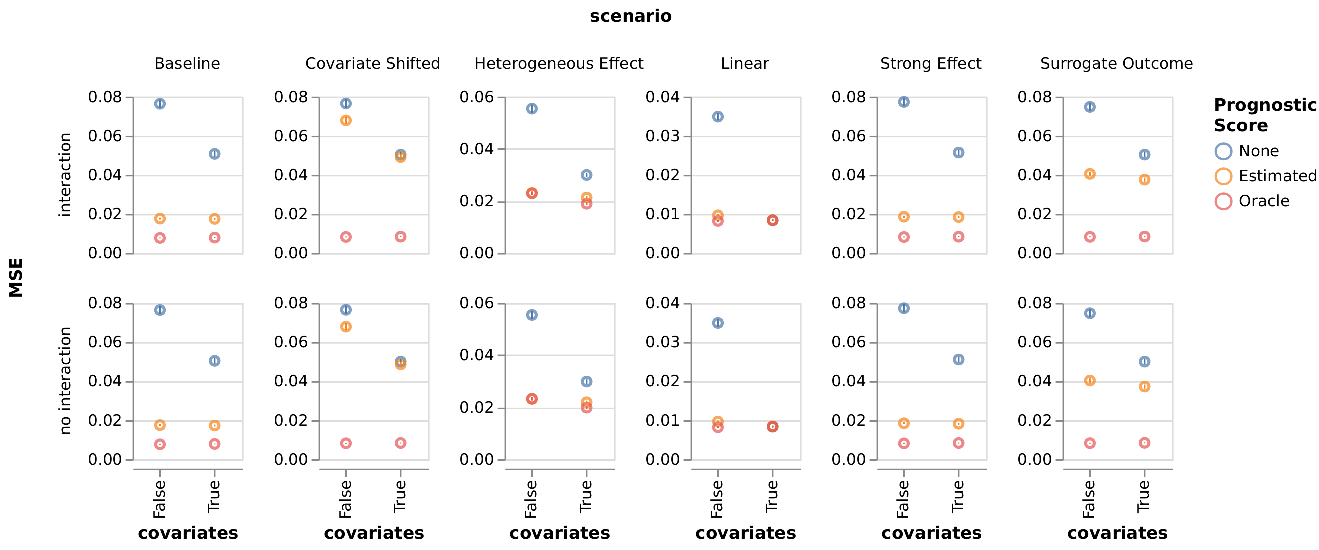}
\caption{Visualization of the simulation results presented in tabular form above.}
\end{figure}

\section{Covariates in the empirical demonstration dataset}

\begin{table}[h!]
\centering
\begin{tabular}{ |r|l| } 
\hline
Covariate & Description \\ 
\hline
AChEI or Memantine usage & Whether a subject is using a class of symptomatic Alzheimer's drugs \\
ADAS Commands & Assesses the subject's ability to follow commands \\
ADAS Comprehension & Assesses the subject's ability to understand spoken language \\
ADAS Construction & Assesses the subject's ability to draw basic figures \\
ADAS Ideational & Assesses the subject's ability to carry out a basic task \\
ADAS Naming & Assesses the subject's ability to name common objects \\
ADAS Orientation & Assesses the subject's knowledge of time and place \\
ADAS Remember Instructions & Assesses the subject's ability to remember test instructions \\
ADAS Spoken Language & Assesses the subject's ability to speak clearly \\
ADAS Word Finding & Assesses the subject's word finding in speech \\
ADAS Word Recall & Assesses the subject's ability to recall a list of words \\
ADAS Word Recognition & Assesses the subject's ability to remember and identify words \\
Age & Subject age at baseline \\
ApoE e4 Allele Count & The number of ApoE e4 alleles a subject has (0, 1, or 2) \\
CDR Community & Assesses the subject's engagement in community activities \\
CDR Home and Hobbies & Assesses the subject's engagement in home and personal activities \\
CDR Judgement & Assesses the subject's judgement skills \\
CDR Memory & Assesses the subject's memory \\
CDR Orientation & Assesses the subject's knowledge of time and place \\
CDR Personal Care & Assesses the subject's ability to care for themselves \\
Diastolic blood pressure & The diastolic blood pressure of a subject \\
Education (Years) & The number of years of education of a subject \\
Heart Rate & The resting heart rate of a subject \\
Height & The height of a subject \\
Indicator for Clinical Trial & 1 if the subject is in an RCT, 0 if not \\
MMSE Attention and Calculation & Assesses the subject's attention and calculation skills \\
MMSE Language & Assesses the subject's language skills \\
MMSE Orientation & Assesses the subject's knowledge of place and time \\
MMSE Recall & Assesses the subject's ability to remember prompts \\
MMSE Registration & Assesses the subject's ability to repeat prompts \\
Region: Europe & 1 if the subject lives in Europe, 0 otherwise \\
Region: Northern America & 1 if the subject lives in the US or Canada, 0 otherwise \\
Region: Other & 1 if the subject lives outside of Europe / US / Canada, 0 otherwise \\
Serious Adverse Events & The number of serious adverse events reported \\
Sex & 1 if female, 0 if male \\
Systolic Blood Pressure & The systolic blood pressure of a subject \\
Weight & The weight of a subject \\
\hline
\end{tabular}
\caption{Baseline covariates in the DHA study and ADNI/CPAD historical training data.}
\end{table}

\end{document}